\documentclass{article}

\usepackage[preprint]{neurips_2026}

\usepackage[T1]{fontenc}    
\usepackage{hyperref}
\usepackage{url}
\usepackage{booktabs}
\usepackage{amsfonts}
\usepackage{amsmath,amssymb,amsthm}
\usepackage{nicefrac}
\usepackage[table]{xcolor}
\usepackage{graphicx}
\usepackage{svg}
\usepackage{caption}
\usepackage{subcaption}
\usepackage{algorithm}
\usepackage{algpseudocode}
\usepackage{enumitem}
\usepackage{xspace}
\usepackage{tikz}
\usetikzlibrary{positioning,arrows.meta,shapes.geometric,fit,calc,backgrounds,decorations.pathreplacing}
\newtheorem{proposition}{Proposition}
\newtheorem{remark}{Remark}

\newcommand{\pref}{\pi_{\mathrm{ref}}}
\newcommand{\ppol}{\pi_\theta}
\newcommand{\pold}{\pi_{\theta_{\mathrm{old}}}}
\newcommand{\pteach}{p_{\mathrm{T}}}
\newcommand{\Zh}{\widehat{Z}}
\newcommand{\Dz}{\mathcal{D}_{\mathrm{Z}}}
\newcommand{\Dtr}{\mathcal{D}_{\mathrm{train}}}
\DeclareMathOperator*{\Lis}{logsumexp}
\newcommand{\KL}{\mathrm{KL}}
\newcommand{\R}{\mathbb{R}}
\newcommand{\E}{\mathbb{E}}

\DeclareRobustCommand{\method}{\TextOrMath{DISA\xspace}{\mathrm{DISA}}}

\title{DISA: Offline Importance Sampling for Distribution-Matching LLM-RL}

\definecolor{author-color}{RGB}{85, 133, 202}

\author{
  {\bf
  Shaobo Wang$^{{\color{author-color}\boldsymbol{1,2}}}$$^{* \dagger}$ \quad
    Yujie Chen$^{{\color{author-color}\boldsymbol{1,3}}}$\thanks{Equal contribution.} \quad
    Yafeng Sun$^{{\color{author-color}\boldsymbol{1,4}}}$ \quad
    Wenjie Qiu$^{{\color{author-color}\boldsymbol{2,5}}}$  \quad Zhihui Xie$^{{\color{author-color}\boldsymbol{2,6}}}$ 
  } \\ 
  {\bf
  Sihang Li$^{{\color{author-color}\boldsymbol{2,4}}}$ \quad
  Yucheng Li$^{{\color{author-color}\boldsymbol{2}}}$ \quad
  Huiqiang Jiang$^{{\color{author-color}\boldsymbol{2}}}$ \quad
  Xingzhang Ren$^{{\color{author-color}\boldsymbol{2}}}$ \quad
  Xuming Hu$^{{\color{author-color}\boldsymbol{3}}}$   
  } \\ \vspace{5pt}
  {\bf \hspace{3pt}
  Dayiheng Liu$^{{\color{author-color}\boldsymbol{2}}}$ \quad
    Linfeng Zhang$^{{\color{author-color}\boldsymbol{1}}}$\thanks{Corresponding authors.}
  } \\ \vspace{1pt}
  {
  $^{\color{author-color}\boldsymbol{1}}$Shanghai Jiao Tong University \quad 
  $^{\color{author-color}\boldsymbol{2}}$Qwen Team, Alibaba Group
  } \\ \vspace{0pt}
  {
  $^{\color{author-color}\boldsymbol{3}}$The Hong Kong University of Science and Technology (Guangzhou)
  } \\ \vspace{0pt}
  {$^{\color{author-color}\boldsymbol{4}}$The University of Science and Technology of China } \\ \vspace{0pt}
  {
  $^{\color{author-color}\boldsymbol{5}}$Nanjing University \quad
  $^{\color{author-color}\boldsymbol{6}}$The University of Hong Kong
  } 
}

\begin{document}

\maketitle

\begin{abstract}
Modern reasoning agents are increasingly evaluated on their ability to generate multiple valid solution paths, plans, or tool-use traces for a given input. Standard reward-maximizing RL tends to collapse onto the most easily reinforced high-reward mode, whereas distribution-matching RL aims to allocate probability mass across the entire reward-shaped solution set. Achieving this objective requires computing a prompt-dependent partition function over the trajectory space. Because existing distribution-matching methods learn this partition function online alongside the policy, calibration errors in the partition function directly distort policy updates and remain impossible to diagnose independently. We introduce \method, short for \emph{Decoupled Importance-Sampled Anchoring}, which moves this calibration problem outside the RL loop. \method draws proposal trajectories offline, estimates the partition function via importance sampling, and freezes the resulting partition-function estimate before policy optimization begins. This decoupling preserves the distribution-matching objective while strictly separating partition-function estimation from policy learning in data, gradients, loss, and diagnostics. Theoretically, our analysis shows that an exact frozen partition function preserves the reward-tilted optimum, while residual offline estimation error enters through a bounded perturbation envelope. Empirically, on two open-weight backbones across six math and three code benchmarks, \method matches or exceeds the online-coupled distribution-matching baseline FlowRL, outperforms reward-maximization baselines GRPO and GSPO on math averages, and exceeds LoRA-SFT distillation by up to $13.8$ Mean@8 points on the same offline trajectories. An LLM-as-judge evaluation further shows that \method retains substantially more strategy-level diversity than reward-maximization baselines, and sensitivity studies on the proposal strength and inverse temperature follow the bias-variance pattern predicted by the analysis.
\end{abstract}

\section{Introduction}
\label{sec:intro}

Reasoning LLMs and LLM-based agents increasingly produce long trajectories: chain-of-thought solutions span many steps~\citep{cot_wei2022}, agentic systems add planning and tool use~\citep{arpo2025}, and inference often aggregates multiple samples per prompt~\citep{sc_wang2023,letsverify2024,tot_yao2023,rstar2024}. In these settings, a post-trained model should not only place high probability on one answer; it should produce \emph{several distinct correct trajectories} for the same prompt. This calls for post-training objectives that recover the reward landscape while preserving diversity across repeated samples~\citep{flowrl2025,prolongedrl2025,entropymechanism2025}.

\begin{figure}[tb!]
\centering
\includegraphics[width=\textwidth]{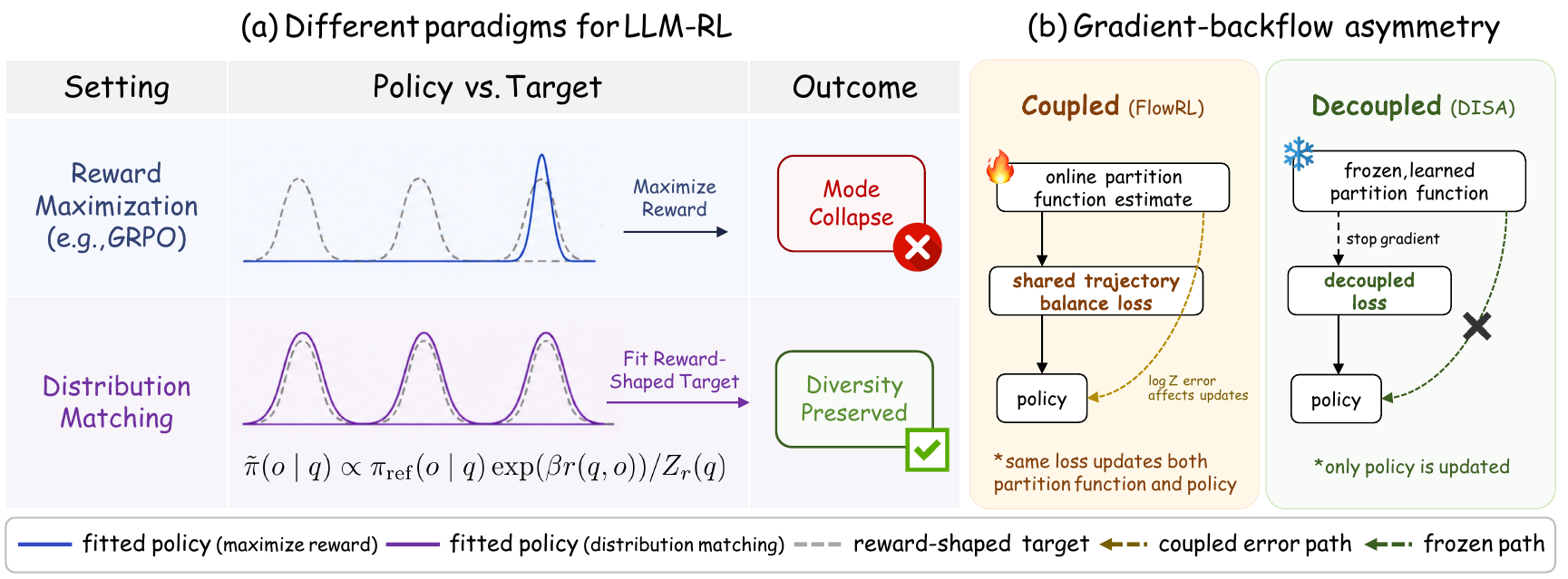}
\caption{\textbf{Two families of post-training and the gradient-backflow
asymmetry that distinguishes them.} \textbf{(a)} On a multi-modal reward
$r(q,\cdot)$, \emph{reward maximization} (PPO/GRPO/GSPO) collapses to a
single mode, while \emph{distribution matching} targets
$\pref e^{\beta r}/Z(q)$ and preserves multi-modal structure.
\textbf{(b)} Prior methods co-train the partition function $Z_\phi$ with $\ppol$ on a shared loss,
so partition-function error rides into $\nabla_\theta$. \method replaces $Z_\phi$ with
a frozen partition function $g_{\psi^\star}$ (dashed wall = stop-gradient), so
$\nabla_\theta$ flows from the policy alone.}
\vspace{-0.5cm}
\label{fig:two-families}
\end{figure}

Existing RL objectives split into two families. \emph{Reward maximization} methods such as PPO, GRPO, DAPO, and GSPO update the policy from scalar rewards under KL-style regularization~\citep{ppo2017,grpo2024,dapo2025,gspo2025}; they are effective for single-answer accuracy but can collapse diverse valid reasoning paths onto one mode~\citep{flowrl2025,prolongedrl2025,entropymechanism2025}. \emph{Distribution matching} instead fits a target where trajectory probability is proportional to exponentiated reward~\citep{bengio2021flow,malkin2022trajectory,yu2024flow,tba2025,flowrl2025}, preserving multi-modal solution structure. Its bottleneck is the per-prompt partition function: prior methods co-train this partition function with the policy on a shared online loss, so partition-function error flows directly into the policy gradient and has no separable quality signal.

The key observation is that this partition function is policy-independent. It is determined by the prompt distribution, reference policy, verifier reward, and inverse temperature, so it can be estimated before policy training begins. We estimate the partition function offline by importance sampling (IS)~\citep{owenmc2013,robertcasella2004,tokdar2010importance,elvira2021advances}, using a fixed proposal that covers reward-bearing trajectories~\citep{owenzhou2000safe}. This turns the per-prompt partition function into a reusable offline estimate, separated from the policy in data, gradients, loss, and diagnostics.

We propose \emph{Decoupled Importance-Sampled Anchoring} (\method), an offline-then-online realization of distribution-matching RL. Stage~1 draws proposal trajectories and forms log-space IS labels for the partition function. Stage~2 fits a lightweight prompt regressor to these labels and freezes it as a partition-function estimate. Stage~3 uses the frozen partition function inside the trajectory-balance objective, with no gradient through the regressor. Appendix~\ref{app:proofs} gives the formal guarantees: an exact frozen partition function preserves the reward-tilted optimum, prompt-only partition-function bias preserves it as an on-policy stationary point, and residual partition-function error enters through a bounded loss-perturbation envelope. Our contributions are as follows:

\begin{itemize}[leftmargin=*, topsep=2pt, itemsep=1pt, parsep=0pt, partopsep=0pt]
\item We identify policy--partition-function coupling as the core obstacle in distribution-matching RL for LLM post-training: it injects partition-function error into the policy gradient and leaves the partition function without an independent quality signal.
\item We propose \method, a three-stage pipeline that estimates the partition function offline, amortizes it into a frozen prompt-conditioned partition function, and uses that partition function in online trajectory-balance RL. We prove exact-partition-function preservation, biased-partition-function stationarity, and a bounded partition-function-error perturbation in Appendix~\ref{app:proofs}.
\item We validate \method against GRPO, GSPO, and FlowRL under matched backbones, corpora, and RL infrastructure on six math benchmarks, including AIME~2024/2025, AMC~2023, MATH-500, Minerva, and OlympiadBench, and three code benchmarks, namely LiveCodeBench, CodeForces, and HumanEval+. On Qwen3-4B-Base, \method matches or exceeds FlowRL on math, with $65.5$ vs.\ $64.8$ Mean@8, exceeds GRPO by $16.1$ points, and achieves the strongest code pass@16, with $31.4$ vs.\ FlowRL $30.9$, GSPO $29.9$, and GRPO $27.7$. An AIME LLM-as-judge evaluation shows that \method retains the largest fraction of the backbone's strategy diversity, with $3.72$ on a $1$-to-$5$ scale compared with FlowRL $3.40$ and GRPO/GSPO $\approx 3.0$.
\end{itemize}

\section{Method: Decoupled Importance-Sampled Anchoring (\texorpdfstring{\method}{DISA})}
\label{sec:method}

\begin{figure}[tb!]
\centering
\includegraphics[width=\textwidth]{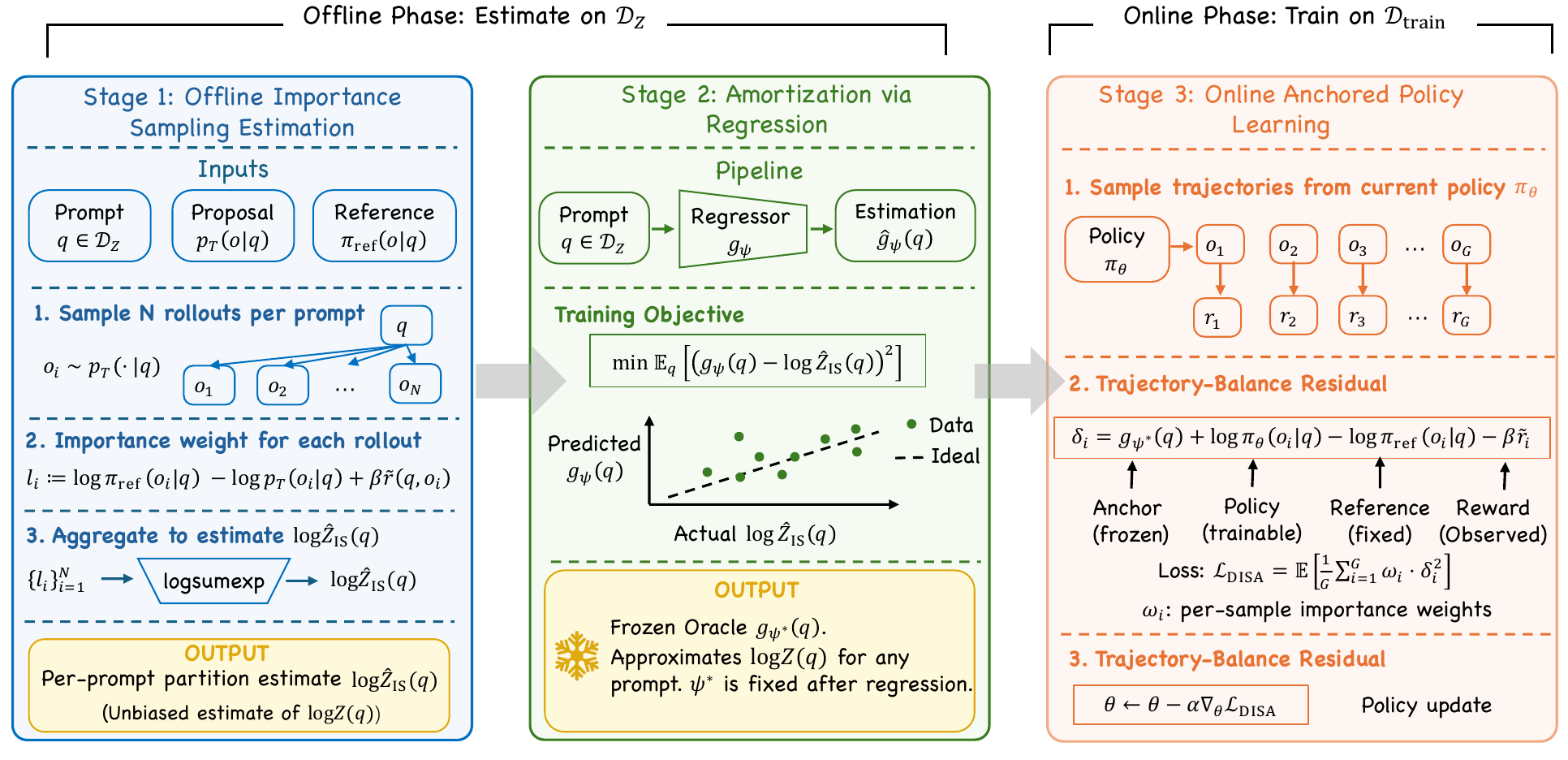}
\caption{\textbf{The pipeline of DISA}, which includes three stages. \emph{Stage~1, offline IS estimation:} for each prompt $q$, draw $N$ trajectories from a proposal $\pteach$ and form the per-prompt label $\log\Zh_{\mathrm{IS}}(q)$ via the logsumexp aggregator. \emph{Stage~2, amortization:} fit a regressor $g_\psi$ to these labels by least squares. \emph{Stage~3, anchored RL:} freeze $g_{\psi^\star}$ and substitute it for $\log Z_\phi$ in the trajectory-balance loss, so the partition function enters the policy update as a per-prompt constant with no gradient through $\psi$.}
\label{fig:pipeline}
\end{figure}

\subsection{Background}
\label{sec:method:bg}

\noindent\textbf{Notation and setup.} Throughout the paper, $q\in\mathcal{Q}$ denotes a prompt and $o\in\mathcal{O}$ a complete output trajectory. A policy $\pi(o\mid q)$ assigns probabilities to trajectories. We are given a reference policy $\pref$, taken in practice to be the supervised-fine-tuned model, and a verifiable binary reward $r:\mathcal{Q}\times\mathcal{O}\to\{0,1\}$ (math-verify for math, execution checks for code) that scores a trajectory once it is complete. The goal of RL post-training is to update the trainable policy $\ppol=\pi_\theta$ from $\pref$ using the reward signal.

\noindent\textbf{Reward maximization: PPO and GRPO.} Let a sampled trajectory be $o_i=(y_{i,1},\ldots,y_{i,T_i})$. PPO~\citep{ppo2017} does \emph{not} use a trajectory-level probability ratio; it clips token-level likelihood ratios
\begin{equation}
\rho_{i,t} \;:=\; \frac{\ppol(y_{i,t}\mid q,y_{i,<t})}{\pold(y_{i,t}\mid q,y_{i,<t})},
\label{eq:bg:ratio}
\end{equation}
where $\pold$ is the policy iterate used to collect the rollouts, and maximizes the token-averaged clipped surrogate
\begin{equation}
\mathcal{L}_{\mathrm{PPO}}(\theta) \;=\; \E\!\bigg[\frac{1}{T_i}\sum_{t=1}^{T_i}\min\!\big(\rho_{i,t}A_{i,t},\;\mathrm{clip}(\rho_{i,t},1-\epsilon,1+\epsilon)A_{i,t}\big)\bigg],
\label{eq:bg:ppo}
\end{equation}
where $A_{i,t}$ is a per-token advantage and $\epsilon>0$ is the clip half-width.
GRPO~\citep{grpo2024} keeps the same token-level clipping but replaces PPO's value-network advantage with a group-normalized trajectory reward. For each prompt $q$, a group of $G$ trajectories $\{o_i\}_{i=1}^G$ is sampled from $\pold$, and the rewards $r_i:=r(q,o_i)$ are normalized within the group,
\begin{equation}
\hat r_i \;=\; \frac{r_i - \bar r_q}{s_q},
\qquad
\bar r_q = \tfrac{1}{G}\textstyle\sum_i r_i,\quad
s_q^2 = \tfrac{1}{G}\textstyle\sum_i (r_i-\bar r_q)^2.
\label{eq:bg:grpo-rhat}
\end{equation}
The scalar $\hat r_i$ is then used as the advantage for all tokens in trajectory $o_i$. Thus the reward-maximization family uses token-level policy-ratio clipping together with scalar trajectory rewards; it does not explicitly fit a normalized distribution over complete trajectories. In the formal distribution-matching statements below, we write $\tilde r(q,o)=a_q r(q,o)+b_q$ for a fixed prompt-conditional affine transform of the verifier reward, with $a_q>0$. The finite group-normalized $\hat r_i$ used in implementation is the plug-in counterpart of this fixed transform.

\begin{remark}[Mode collapse of reward maximization]\label{rmk:mode-collapse}
PPO and GRPO follow a scalar reward gradient under a clip-stabilized off-policy surrogate; they never instantiate the reward-tilted distribution as an explicit target. The reported empirical consequence is that probability mass concentrates on a small subset of high-reward modes, eliminating the diversity of correct trajectories that downstream pass@$k$ and multi-sample evaluation rely on~\citep{flowrl2025,prolongedrl2025,entropymechanism2025}.
\end{remark}

\noindent\textbf{Distribution matching: GFlowNets and FlowRL.} A distribution-matching alternative replaces reward maximization with a fitting target. Fix an inverse temperature $\beta>0$. For a fixed reward transform $\tilde r$, the reward-tilted distribution is
\begin{equation}
\tilde\pi(o\mid q) \;=\; \frac{1}{Z(q)}\,\pref(o\mid q)\exp(\beta \tilde r(q,o)),
\qquad
Z(q) \;=\; \sum_{o\in\mathcal{O}} \pref(o\mid q)\exp(\beta \tilde r(q,o)).
\label{eq:m:target}
\end{equation}
The per-prompt scalar $Z(q)$ is the partition function. Because $\tilde r(q,o)=a_q r(q,o)+b_q$, the target in~\eqref{eq:m:target} is the usual Boltzmann target with a prompt-dependent effective inverse temperature $\beta_{\mathrm{eff}}(q)=a_q\beta$; the additive shift $b_q$ is absorbed into $Z(q)$.

GFlowNets~\citep{bengio2021flow,malkin2022trajectory} fit a target of this form by jointly training the policy $\ppol$ and a learned partition-function model $Z_\phi(q)$ with parameters $\phi$, via the \emph{trajectory-balance} (TB) loss,
\begin{equation}
\mathcal{L}_{\mathrm{TB}}(\theta,\phi) \;=\; \E_{(q,o)}\!\Big[\big(\log Z_\phi(q) + \log\ppol(o\mid q) - \log\pref(o\mid q) - \beta\, \tilde r(q,o)\big)^2\Big].
\label{eq:bg:tb}
\end{equation}
With the exact $\log Z(q)$, the residual vanishes at the target distribution in~\eqref{eq:m:target}. FlowRL~\citep{flowrl2025} adapts this residual to LLM reasoning by using normalized rewards, evaluating residuals on online rollouts, and learning the prompt partition function online. To state the distribution-matching target without conflating it with PPO's token-level clipping, we write the formal trajectory-level objective with a nonnegative rollout weight $\omega_i$:
\begin{equation}
\mathcal{L}_{\mathrm{FlowRL}}(\theta,\phi) \;=\; \E_{(q,\{o_i\})}\!\bigg[\,\frac{1}{G}\sum_{i=1}^G \omega_i\,\Big(\log Z_\phi(q) + \log\ppol(o_i\mid q) - \log\pref(o_i\mid q) - \beta\,\tilde r(q,o_i)\Big)^2\bigg].
\label{eq:bg:flowrl}
\end{equation}
Here $\omega_i$ is a rollout-level weighting factor equal to $1$ in the on-policy case; it is not the PPO token ratio in~\eqref{eq:bg:ratio}. Some implementations rescale the log-likelihood ratio by trajectory length for numerical conditioning; this paper's formal guarantees are stated for the trajectory-level residual in~\eqref{eq:bg:flowrl}, which is the residual for which a scalar partition function has the Boltzmann fixed point in~\eqref{eq:m:target}.

\begin{remark}[Policy--partition-function coupling]\label{rmk:coupling}
In FlowRL, $\theta$ and $\phi$ are updated jointly by SGD on~\eqref{eq:bg:flowrl} against the same online rollouts. Because both gradients descend the same scalar squared residual, estimation error in $Z_\phi$ and the policy gradient share a single channel: partition-function error propagates directly into the policy update, and the partition function admits no quality signal separable from the policy's loss. \method therefore addresses the following design question: estimate $\log Z(q)$ outside the policy loop, with sufficient accuracy that the exact partition-function target is preserved, and with sufficient stability that residual estimation error enters the policy objective only through a controlled perturbation.
\end{remark}

\subsection{Decoupled Importance-Sampled Anchoring (\method)}
\label{sec:method:disa}

\noindent\textbf{Design principles.} The formal partition function $Z(q)$ in~\eqref{eq:m:target} depends only on $\pref$ and the fixed transformed reward $\tilde r$, not on $\ppol$; it is therefore policy-independent and can be estimated outside the RL loop. \method instantiates this observation through three stages: an offline importance-sampling estimator that produces a per-prompt label $\log\Zh_{\mathrm{IS}}(q)$, a small regressor that amortizes these labels into a function of $q$, and an RL stage that freezes the regressor and substitutes it into the TB residual in place of $\log Z_\phi(q)$. The exact guarantees are stated for the fixed-transform target above; finite-sample plug-in errors from using batch-normalized rewards are treated as approximation error in the offline labels and in the frozen regressor, as formalized in App.~\ref{app:proof:reg-bound}; Algorithm~\ref{alg:decrl} in App.~\ref{app:method:algorithm} summarizes the full offline--online procedure.

\noindent\textbf{Stage 1: importance-sampled estimation of the partition function.} Let $\Dz$ denote the offline prompt set used for partition-function estimation; in practice we take $\Dz=\Dtr$, the RL training prompts (\S\ref{sec:experiments:setup}). For each prompt $q\in\Dz$, DISA draws $N$ samples from a proposal model $\pteach(\cdot\mid q)$ and evaluates their rewards and log-probabilities under both $\pref$ and $\pteach$. Direct Monte Carlo with $o_i\sim\pref$ is unbiased on the linear scale but statistically degenerate on difficult reasoning problems because $\pref$ rarely produces reward-bearing trajectories. We therefore use a proposal model that places more mass on the high-reward tail. For any $\pteach$ satisfying the condition in App.~\ref{app:proof:unbiased},
\begin{equation}
Z(q) \;=\; \E_{o\sim\pteach}\!\bigg[\frac{\pref(o\mid q)}{\pteach(o\mid q)}\,\exp(\beta\, \tilde r(q,o))\bigg],\qquad
\Zh_{\mathrm{IS}}(q) \;=\; \frac1N\sum_{i=1}^{N} w(q,o_i),
\label{eq:m:is}
\end{equation}
with $w(q,o) := (\pref/\pteach)\exp(\beta \tilde r(q,o))$ and $o_i\sim\pteach$. For numerical stability, we evaluate the estimator in log-space:
\begin{equation}
\log\Zh_{\mathrm{IS}}(q) \;=\; \Lis_i\Big(\log\pref(o_i\mid q)-\log\pteach(o_i\mid q)+\beta\, \tilde r(q,o_i)\Big) - \log N.
\label{eq:m:logsumexp}
\end{equation}
The linear-scale estimator~\eqref{eq:m:is} is unbiased; its log-space form~\eqref{eq:m:logsumexp} carries an $O(1/N)$ Jensen bias and is preferred over the geometric-mean log-aggregator $\tfrac1N\sum_i\log w(q,o_i)$, whose downward bias does not decay with $N$ (Apps.~\ref{app:proof:unbiased},~\ref{app:proof:gm}). The proposal $\pteach$ enters \method only through $w$: swapping $\pteach$ alters the variance of the IS labels but not the target under the absolute-continuity condition.

\noindent\textbf{Stage 2: amortized denoising via a prompt-conditioned regressor.} Stage~1 yields a per-prompt estimate $\log\Zh_{\mathrm{IS}}(q)$ for every $q\in\Dz$. Substituting these labels directly into the TB loss would retain per-prompt sampling noise and would provide no value for prompts outside $\Dz$. We therefore introduce an amortizer $g_\psi:\mathcal{Q}\to\R$ that maps a prompt to a scalar prediction of $\log Z(q)$. We train $g_\psi$ once by least-squares regression,
\begin{equation}
\psi^{\star}=\mathop{\arg\min}_\psi \;\;\E_{q\sim\Dz}\!\Big[\big(g_\psi(q)-\log\Zh_{\mathrm{IS}}(q)\big)^2\Big].
\label{eq:m:reg}
\end{equation}
The held-out validation loss or $R^2$ of this fit is a separable diagnostic of how well the regressor fits the offline IS labels; it is not an online policy-training loss. Only $g_{\psi^\star}$ is retained after this stage, and the proposal rollouts can be discarded.

\noindent\textbf{Stage 3: anchored RL with a frozen partition function.} During RL training, we freeze $\psi$ and substitute $\log Z(q)\leftarrow g_{\psi^{\star}}(q)$ in the TB residual. At the online objective level, the only change from FlowRL in~\eqref{eq:bg:flowrl} is the source of the partition function: FlowRL learns $Z_\phi$ jointly with the policy, whereas \method queries a frozen offline partition-function estimate with no gradient through $\psi$, so the partition function acts as a fixed per-prompt scalar in policy optimization. With the formal fixed transform $\tilde r$, the objective is
\begin{equation}
\mathcal{L}_{\method}(\theta) \;=\; \E_{(q,\{o_i\})}\!\bigg[\,\frac{1}{G}\sum_{i=1}^G \omega_i\,\Big(g_{\psi^{\star}}(q) + \log\ppol(o_i\mid q) - \log\pref(o_i\mid q) - \beta\,\tilde r(q,o_i)\Big)^2\bigg].
\label{eq:m:decrl}
\end{equation}
Formal guarantees are deferred to Appendix~\ref{app:proofs}: with an exact partition function the global minimizer of~\eqref{eq:m:decrl} matches the target~\eqref{eq:m:target} (App.~\ref{app:proof:fixedpoint}); a prompt-only partition-function bias preserves on-policy stationarity but not global minimality (App.~\ref{app:on-policy}); and the partition-function MSE $\sigma_g^2:=\E_q[(g_{\psi^\star}(q)-\log Z(q))^2]$ controls the loss perturbation by a Cauchy--Schwarz envelope $\sigma_g^2+2\sigma_g\sqrt{\mathcal{L}_{\mathrm{TB}}(\theta;Z)}$ (App.~\ref{app:proof:reg-bound}). A more accurate offline partition function tightens this envelope; the exact partition-function target is unchanged.

\section{Experiments}
\label{sec:experiments}


\subsection{Setup}
\label{sec:experiments:setup}

\noindent\textbf{Backbones, baselines, corpora.} We instantiate \method on two open-weight backbones used as the reference policy $\pref$: Qwen2.5-7B~\citep{qwen25_2024} and Qwen3-4B-Base~\citep{qwen3_2025}.\footnote{Qwen2.5-7B is also a base model; the explicit ``-Base'' suffix on Qwen3-4B-Base follows Qwen's release naming.} We compare against three RL baselines run from the same backbone with the same data and RL infrastructure, so all reported gaps are attributable to the objective rather than to the data or the trainer. \textbf{GRPO}~\citep{grpo2024} and \textbf{GSPO}~\citep{gspo2025} represent the reward-maximization family of~\S\ref{sec:related}, with GSPO replacing GRPO's token-level importance ratio by a sequence-level one; \textbf{FlowRL}~\citep{flowrl2025} represents the online-coupled distribution-matching family from which \method differs only along the partition-function source---online co-trained versus offline frozen---making FlowRL the controlled comparison that isolates the effect of decoupling. Training corpora are the DAPO-17k-processed prompt set~\citep{dapo2025} for math and the DeepCoder prompt set~\citep{deepcoder2025} for code; the offline budget $\Dz$ for Stages~1--2 of \method coincides with the RL prompt set $\Dtr$, so $|\Dz|=|\Dtr|$.

\noindent\textbf{Hyperparameters.} All four methods share $G=8$ rollouts per prompt, and \method and FlowRL share the inverse-temperature $\beta=15$ on both domains. The Stage~1 proposal $\pteach$ is Qwen3-235B-A22B-Instruct-2507~\citep{qwen3_2025} for the main results, replaced by the smaller Qwen3-4B-Instruct-2507 in the proposal-strength ablation of~\S\ref{sec:experiments:robust}; we draw $N=8$ teacher samples per prompt, matching $G$ and identified as the efficiency--accuracy sweet spot by the variance--bias study in App.~\ref{app:exp:nstudy}, and aggregate them into $\log\Zh_{\mathrm{IS}}(q)$ via the \texttt{logsumexp} estimator of~\eqref{eq:m:logsumexp}. The Stage~2 regressor $g_\psi$ is a small MLP fitted by least squares to those labels. Full optimization, decoding, and \method-specific hyperparameters, together with the per-method checkpoint-selection rule that is applied identically across methods, are in Apps.~\ref{app:exp:hparams},~\ref{app:exp:disa}.

\noindent\textbf{Evaluation.} For mathematical reasoning we evaluate on AIME 2024, AIME 2025~\citep{maa2024aime,maa2025aime}, AMC 2023~\citep{maa2023amc}, MATH-500~\citep{letsverify2024}, Minerva~\citep{minerva2022}, and OlympiadBench~\citep{olympiadbench2024}, reporting \textbf{Mean@8} via the math-verify rule-based grader. For code generation we evaluate on LiveCodeBench~\citep{livecodebench2024}, CodeForces~\citep{codeforces2025}, and HumanEval+~\citep{humaneval2021} with execution-based scoring, reporting both \textbf{pass@1} and \textbf{pass@16}; pass@16 is the relevant metric for the diversity-of-correct-trajectories goal of~\S\ref{sec:intro}, since it rewards placing probability mass on multiple distinct correct programs.

\subsection{Main Results}
\label{sec:experiments:main}

\begin{table}[tb!]
\centering
\small
\vspace{-10pt}
\caption{\textbf{Mathematical reasoning, Mean@8} on AIME 2024, AIME 2025, AMC 2023, MATH-500, Minerva, and OlympiadBench for Qwen2.5-7B and Qwen3-4B-Base. Superscripts in green/red give the delta vs.\ the Vanilla baseline within each backbone block; \textbf{bold} marks the best within-block, \underline{underline} the runner-up.}
\vspace{-10pt}
\setlength{\tabcolsep}{4pt}
\providecommand{\dup}[1]{\textsuperscript{\textcolor{green!55!black}{$+$#1}}}
\providecommand{\ddown}[1]{\textsuperscript{\textcolor{red!75!black}{$-$#1}}}
\begin{tabular}{l|l|llllll}
\toprule
Method & Avg. & AIME24 & AIME25 & AMC23 & MATH500 & Minerva & Olympiad\\
\midrule
\rowcolor{black!8}
\multicolumn{8}{l}{\textit{Backbone: Qwen2.5-7B}}\\
\;\; Vanilla
    & 23.4 & 4.58 & 3.33 & 31.6 & 54.7 & 22.2 & 23.9\\
\;\; \,+\,GRPO
    & 32.8\dup{9.4}
    & \underline{13.2}\dup{8.6}
    & 7.92\dup{4.6}
    & \textbf{61.3}\dup{29.7}
    & 63.1\dup{8.4}
    & 22.4\dup{0.2}
    & 29.1\dup{5.2}\\
\;\; \,+\,GSPO
    & 29.3\dup{5.9}
    & \textbf{13.3}\dup{8.7}
    & 4.17\dup{0.8}
    & 48.1\dup{16.5}
    & 58.5\dup{3.8}
    & 22.7\dup{0.5}
    & 28.8\dup{4.9}\\
\;\; \,+\,FlowRL
    & \underline{33.2}\dup{9.8}
    & 10.0\dup{5.4}
    & \underline{9.17}\dup{5.8}
    & 50.6\dup{19.0}
    & \underline{62.7}\dup{8.0}
    & \textbf{34.1}\dup{11.9}
    & \textbf{32.4}\dup{8.5}\\
\;\; \,+\,\method(Ours)
    & \textbf{33.5}\dup{10.1}
    & 10.0\dup{5.4}
    & \textbf{13.3}\dup{10.0}
    & \underline{52.5}\dup{20.9}
    & \textbf{65.4}\dup{10.7}
    & \underline{27.9}\dup{5.7}
    & \underline{31.9}\dup{8.0}\\
\midrule
\rowcolor{black!8}
\multicolumn{8}{l}{\textit{Backbone: Qwen3-4B-Base}}\\
\;\; Vanilla
    & 29.7 & 7.92 & 6.25 & 42.6 & 64.4 & 26.7 & 30.1\\
\;\; \,+\,GRPO
    & 49.4\dup{19.7}
    & 33.3\dup{25.4}
    & 40.0\dup{33.8}
    & 70.0\dup{27.4}
    & 76.6\dup{12.2}
    & 29.0\dup{2.3}
    & 47.5\dup{17.4}\\
\;\; \,+\,GSPO
    & 62.7\dup{33.0}
    & \underline{66.9}\dup{59.0}
    & 52.1\dup{45.9}
    & 83.5\dup{40.9}
    & \underline{78.2}\dup{13.8}
    & \textbf{38.0}\dup{11.3}
    & \underline{57.2}\dup{27.1}\\
\;\; \,+\,FlowRL
    & \underline{64.8}\dup{35.1}
    & 66.7\dup{58.8}
    & \underline{53.3}\dup{47.1}
    & \underline{92.5}\dup{49.9}
    & \textbf{79.4}\dup{15.0}
    & 36.8\dup{10.1}
    & \textbf{60.1}\dup{30.0}\\
\;\; \,+\,\method(Ours)
    & \textbf{65.5}\dup{35.8}
    & \textbf{67.1}\dup{59.2}
    & \textbf{59.2}\dup{53.0}
    & \textbf{95.3}\dup{52.7}
    & 77.2\dup{12.8}
    & \underline{37.4}\dup{10.7}
    & 56.8\dup{26.7}\\
\bottomrule
\end{tabular}
\vspace{-10pt}
\label{tab:math:main}
\end{table}

\begin{table}[tb!]
\centering
\small
\caption{\textbf{Code generation, pass@1 and pass@16} on three benchmarks for Qwen2.5-7B and Qwen3-4B-Base. Superscripts in green show the absolute change of each $+$method relative to the Vanilla baseline of the same backbone; no regressions are observed in the code domain. \textbf{Bold} marks the best value within a backbone block under each metric; \underline{underline} marks the runner-up.}
\vspace{-10pt}
\setlength{\tabcolsep}{2pt}
\providecommand{\dup}[1]{\textsuperscript{\textcolor{green!55!black}{$+$#1}}}
\providecommand{\ddown}[1]{\textsuperscript{\textcolor{red!75!black}{$-$#1}}}
\begin{tabular}{l|ll|llllll}
\toprule
& \multicolumn{2}{c|}{Avg.} & \multicolumn{2}{c}{LiveCodeBench} & \multicolumn{2}{c}{CodeForces} & \multicolumn{2}{c}{HumanEval+}\\
\cmidrule(lr){2-3}\cmidrule(lr){4-5}\cmidrule(lr){6-7}\cmidrule(lr){8-9}
Method & pass@1 & pass@16 & pass@1 & pass@16 & pass@1 & pass@16 & pass@1 & pass@16\\
\midrule
\rowcolor{black!8}
\multicolumn{9}{l}{\textit{Backbone: Qwen2.5-7B}}\\
\;\; Vanilla
    & 3.81 & 25.8 & 5.60 & 21.5 & 1.38 & 11.8 & 4.45 & 44.2\\
\;\; \,+\,GRPO
    & \textbf{32.98}\dup{29.2}
    & 45.0\dup{19.2}
    & \textbf{16.85}\dup{11.3}
    & 23.4\dup{1.9}
    & \underline{9.07}\dup{7.7}
    & 20.8\dup{9.0}
    & \textbf{73.01}\dup{68.6}
    & \underline{90.8}\dup{46.6}\\
\;\; \,+\,GSPO
    & 30.17\dup{26.4}
    & 44.2\dup{18.4}
    & 14.40\dup{8.8}
    & \textbf{26.2}\dup{4.7}
    & 7.02\dup{5.6}
    & \textbf{23.0}\dup{11.2}
    & 69.10\dup{64.7}
    & 83.4\dup{39.2}\\
\;\; \,+\,FlowRL
    & 31.60\dup{27.8}
    & \textbf{46.2}\dup{20.4}
    & 14.70\dup{9.1}
    & 24.5\dup{3.0}
    & \textbf{10.78}\dup{9.4}
    & 21.5\dup{9.7}
    & 69.33\dup{64.9}
    & \textbf{92.6}\dup{48.4}\\
\;\; \,+\,\method(Ours)
    & \underline{31.80}\dup{28.0}
    & \underline{45.3}\dup{19.5}
    & \underline{15.41}\dup{9.8}
    & \underline{25.4}\dup{3.9}
    & 8.82\dup{7.4}
    & \underline{22.8}\dup{11.0}
    & \underline{71.17}\dup{66.7}
    & 87.7\dup{43.5}\\
\midrule
\rowcolor{black!8}
\multicolumn{9}{l}{\textit{Backbone: Qwen3-4B-Base}}\\
\;\; Vanilla
    & 5.85 & 23.1 & 7.48 & 24.0 & 2.31 & 17.2 & 7.76 & 28.2\\
\;\; \,+\,GRPO
    & \underline{35.33}\dup{29.5}
    & \underline{50.8}\dup{27.7}
    & 19.10\dup{11.6}
    & 30.4\dup{6.4}
    & \textbf{11.00}\dup{8.7}
    & \underline{30.1}\dup{12.9}
    & \underline{75.90}\dup{68.1}
    & \underline{92.0}\dup{63.8}\\
\;\; \,+\,GSPO
    & 35.03\dup{29.2}
    & 49.8\dup{26.7}
    & \textbf{19.60}\dup{12.1}
    & \textbf{31.9}\dup{7.9}
    & 9.90\dup{7.6}
    & 27.5\dup{10.3}
    & 75.60\dup{67.8}
    & 90.1\dup{61.9}\\
\;\; \,+\,FlowRL
    & 35.23\dup{29.4}
    & 49.2\dup{26.1}
    & 18.90\dup{11.4}
    & 28.9\dup{4.9}
    & \underline{10.90}\dup{8.6}
    & 28.6\dup{11.4}
    & 75.90\dup{68.1}
    & 90.1\dup{61.9}\\
\;\; \,+\,\method(Ours)
    & \textbf{35.47}\dup{29.6}
    & \textbf{51.5}\dup{28.4}
    & \underline{19.30}\dup{11.8}
    & \underline{30.8}\dup{6.8}
    & 10.84\dup{8.5}
    & \textbf{31.2}\dup{14.0}
    & \textbf{76.26}\dup{68.5}
    & \textbf{92.6}\dup{64.4}\\
\bottomrule
\end{tabular}
\label{tab:code:main}
\end{table}

\noindent\textbf{Mathematical reasoning.} The two distribution-matching methods \method and FlowRL lead the six-benchmark average on both backbones (Table~\ref{tab:math:main}): \method matches FlowRL on Qwen2.5-7B and is slightly ahead on Qwen3-4B-Base, while GSPO trails closely on the larger backbone and GRPO sits substantially below. This ordering is consistent with the role of the offline partition function as a substitution that should preserve, rather than redesign, the FlowRL fixed point of Prop.~\ref{prop:fixedpoint}. The gap to GRPO is most visible on Qwen3-4B-Base, where \method exceeds GRPO by $16.1$ points on the average and by larger margins on the individual contest benchmarks AIME~2024, AIME~2025, and AMC~2023. We read this gap not as evidence that \method is uniformly stronger on every benchmark---FlowRL takes Minerva on the smaller backbone and OlympiadBench on the larger one, GSPO takes Minerva and is the runner-up on AIME~2024 on the larger backbone, and GRPO takes AMC~2023 on the smaller backbone---but as evidence that fitting an explicit reward-tilted target rather than maximizing reward is the more important design choice on multi-mode mathematical reasoning, and that the offline-then-online split of \method recovers the FlowRL benefit without the partition-coupling cost discussed in~\S\ref{sec:related}.

\noindent\textbf{Code generation.} The four methods separate along the pass@1 vs.\ pass@16 axis predicted by the diversity argument of~\S\ref{sec:intro} (Table~\ref{tab:code:main}). On pass@1, where the metric rewards a single peaked correct program rather than a spread over correct programs and reward maximization is expected to be competitive, GRPO holds a small lead on Qwen2.5-7B---driven mainly by HumanEval+---with \method, FlowRL, and GSPO within $\approx 3$ points; on Qwen3-4B-Base all four methods cluster tightly inside a one-point band, with \method on top followed by GRPO, FlowRL, and GSPO. On pass@16, which rewards exactly the multi-trajectory coverage motivating~\S\ref{sec:intro}, a distribution-matching method leads on both backbones: \method on Qwen3-4B-Base, with GRPO, GSPO, and FlowRL trailing in that order; FlowRL on Qwen2.5-7B, with \method second and GRPO and GSPO behind. Taken together, the substitution of the partition-function source is performance-preserving on FlowRL's math strength and competitive with or ahead of reward maximization on the diversity-relevant pass@16 axis, with no headline regression on pass@1 either.

\subsection{Ablation Studies}
\label{sec:experiments:robust}

We assess two design choices that the offline construction exposes---the strength of the Stage~1 proposal $\pteach$ and the inverse-temperature $\beta$---and which the propositions of~\S\ref{sec:method} predict should affect offline-label quality and the partition-function-error envelope without disturbing the fixed point.

\begin{figure}[t]
\centering
\begin{subfigure}{0.45\textwidth}
  \centering
  \includegraphics[width=\textwidth]{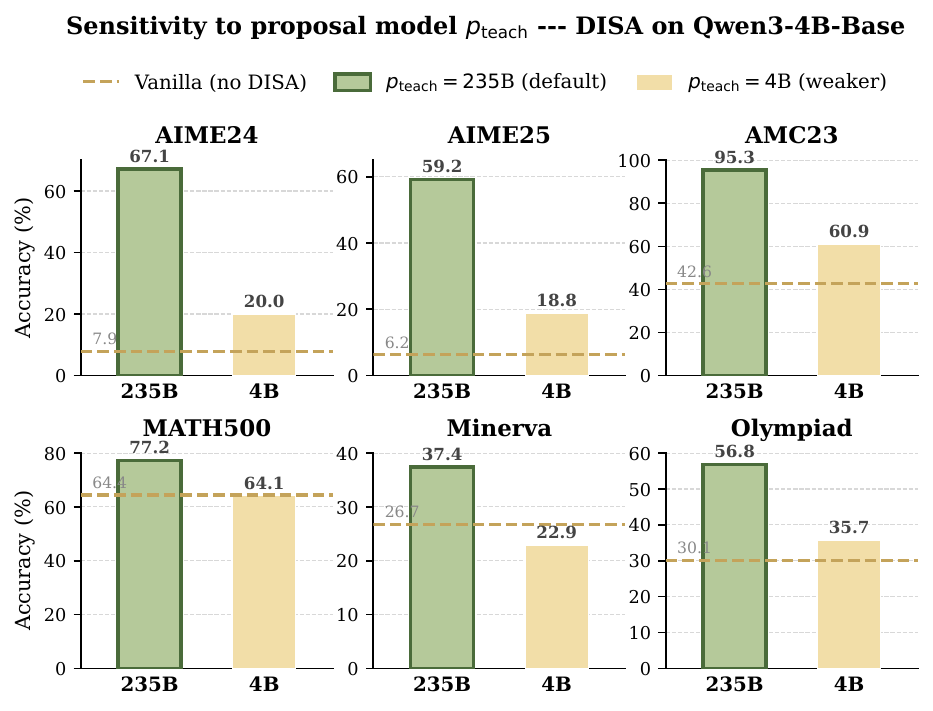}
  \captionsetup{skip=2pt}
  \subcaption{}\label{fig:weakteacher:acc}
\end{subfigure}%
\hfill
\begin{subfigure}{0.55\textwidth}
  \centering
  \includegraphics[width=\textwidth]{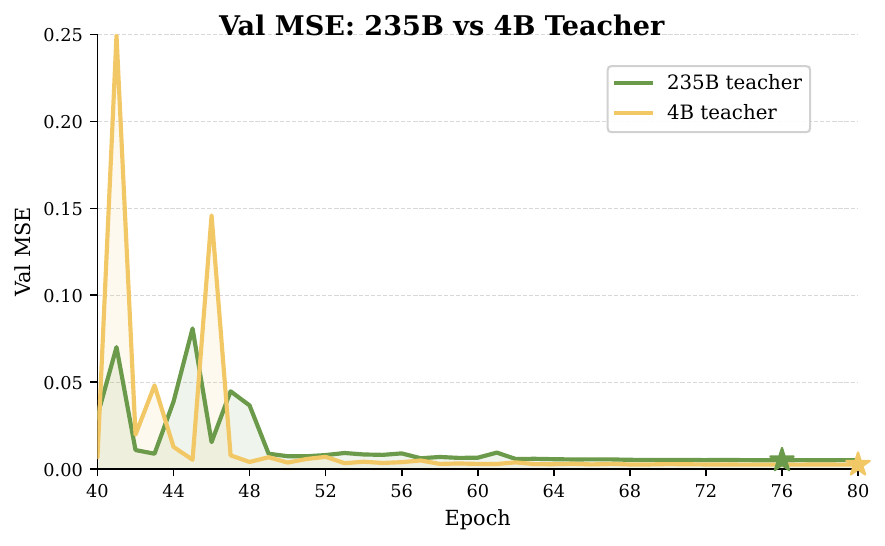}
  \captionsetup{skip=2pt}
  \subcaption{}\label{fig:weakteacher:loss}
\end{subfigure}
\caption{\textbf{Sensitivity to the proposal model $\pteach$.}
\emph{$\pteach$=235B}: default Qwen3-235B-A22B-Instruct-2507; \emph{$\pteach$=4B}: weaker Qwen3-4B-Instruct-2507.
(a)~\method on Qwen3-4B-Base under the two proposals, Mean@8 across six math benchmarks.
(b)~Validation MSE of the Stage~2 regressor $g_\psi$ across training epochs under the two proposals; stars mark the selected checkpoints.}
\vspace{-0.2cm}
\label{fig:weakteacher}
\end{figure}

\noindent\textbf{Sensitivity to $\pteach$.} We replace the default Qwen3-235B-A22B-Instruct-2507 by the $\sim\!60\times$ smaller Qwen3-4B-Instruct-2507 and rerun the entire pipeline on Qwen3-4B-Base. Figure~\ref{fig:weakteacher}(a) shows the Mean@8 average dropping from $65.5$ to $37.1$ while Figure~\ref{fig:weakteacher}(b) shows the Stage~2 regressor's terminal validation MSE remaining comparably low under both proposals. This is exactly the bias--variance pattern of Props.~\ref{prop:unbiased}--\ref{prop:reg-bound}: the weaker proposal leaves $\E[\Zh_{\mathrm{IS}}]=Z(q)$ unchanged but inflates its variance, so the accuracy gap traces to IS-label variance widening the Cauchy--Schwarz envelope rather than to a regression failure, and the exact partition-function target is unchanged. Per-benchmark deltas and the Val-MSE convergence reading are in App.~\ref{app:exp:weak-teacher-details}.

\begin{figure}[t]
\centering
\includegraphics[width=\textwidth]{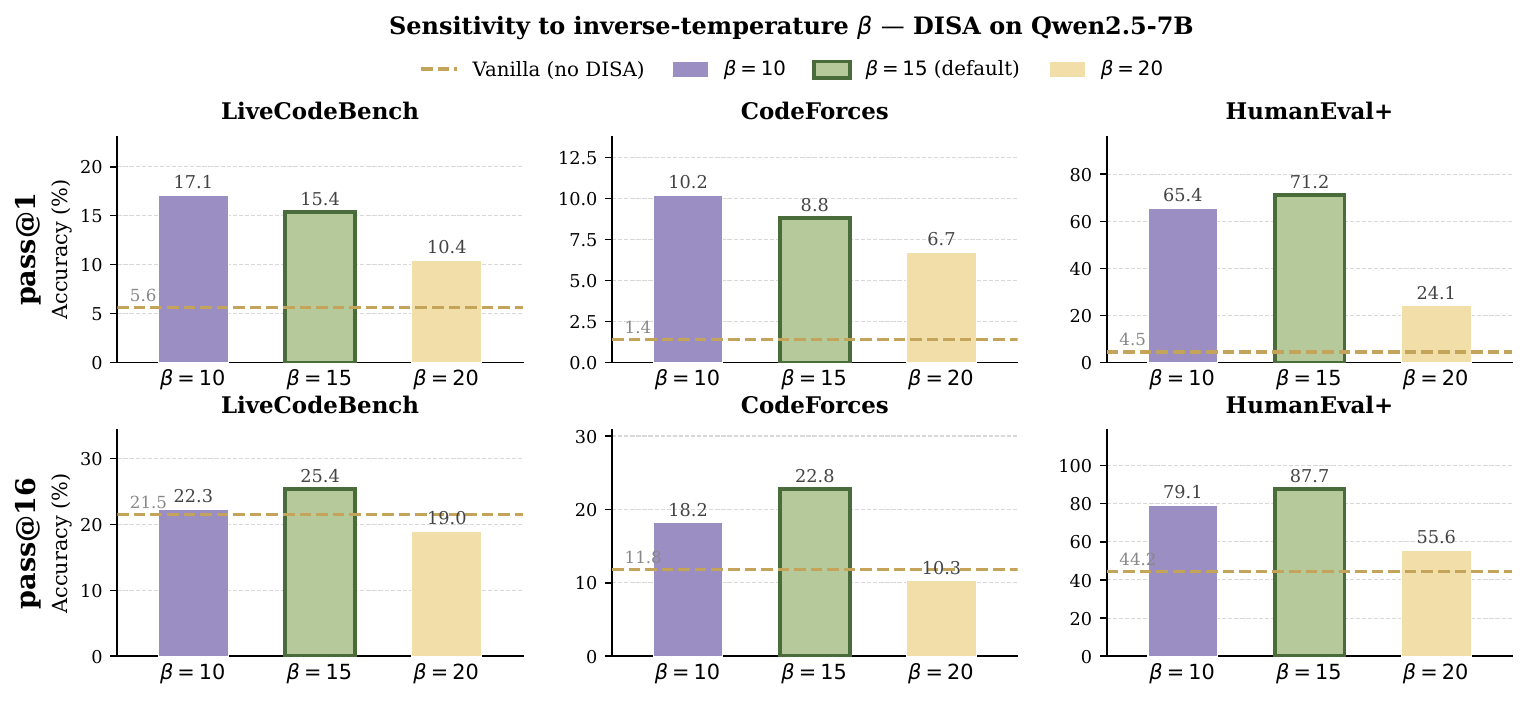}
\caption{\textbf{Sensitivity to the inverse-temperature $\beta$ on Qwen2.5-7B code.}
\method at $\beta\in\{10,15,20\}$, pass@1 on the left and pass@16 on the right, with all other hyperparameters held fixed.
The default $\beta=15$, hatched and starred, matches FlowRL; the untrained backbone is shown in gold.}
\label{fig:beta}
\end{figure}

\noindent\textbf{Sensitivity to $\beta$.} Sweeping $\beta\in\{10,15,20\}$ on Qwen2.5-7B code (Figure~\ref{fig:beta}) yields the asymmetric profile predicted by the $\mathrm{CV}^2(w)$ analysis of Prop.~\ref{prop:unbiased}: $\beta=10$ stays close to the default on pass@1 but already loses ground on pass@16, while $\beta=20$ collapses on both metrics. Under-tilting gives a plateau in which mild signal loss is absorbed by the regression; over-tilting hits a variance cliff once $\mathrm{CV}^2(w)$ becomes large enough for LSE to be dominated by a few high-weight samples, with diversity-relevant pass@16 especially sensitive. The default $\beta=15$ sits inside the plateau and well clear of the cliff, matching FlowRL~\citep{flowrl2025}; per-benchmark numbers at the two non-default settings are in App.~\ref{app:exp:beta-details}.

\section{Discussion}
\label{sec:discussion}


\subsection{Diversity of Solution Strategies}
\label{sec:discussion:diversity}

\begin{table}[tb!]
\centering
\small
\caption{\textbf{Diversity of solution strategies on AIME 2024 + AIME 2025} for Qwen2.5-7B and its RL-post-trained variants. Mean and standard deviation of GPT-4o-mini Likert scores over $60$ problems, with the right two columns breaking the score down by sub-benchmark. Red superscripts give the absolute drop vs.\ the Qwen2.5-7B backbone; \textbf{bold} marks the best within-column value, \underline{underline} the runner-up.}
\vspace{-10pt}
\setlength{\tabcolsep}{8pt}
\providecommand{\dup}[1]{\textsuperscript{\textcolor{green!55!black}{$+$#1}}}
\providecommand{\ddown}[1]{\textsuperscript{\textcolor{red!75!black}{$-$#1}}}
\begin{tabular}{l|l|llll}
\toprule
Method & Diversity $\uparrow$ & Std., $N{=}60$ & AIME 2024 & AIME 2025\\
\midrule
Qwen2.5-7B    & 3.90            & 0.30 & 3.90            & 3.90\\
\;\,+\,GRPO                      & 3.02\ddown{0.88} & 0.72 & 2.93\ddown{0.97} & 3.10\ddown{0.80}\\
\;\,+\,GSPO                      & 2.97\ddown{0.93} & 1.06 & 2.63\ddown{1.27} & 3.30\ddown{0.60}\\
\;\,+\,FlowRL                    & \underline{3.40}\ddown{0.50} & 0.95 & \underline{3.20}\ddown{0.70} & \textbf{3.60}\ddown{0.30}\\
\;\,+\,\method, ours            & \textbf{3.72}\ddown{0.18} & 0.71 & \textbf{3.83}\ddown{0.07} & \textbf{3.60}\ddown{0.30}\\
\bottomrule
\end{tabular}
\label{tab:diversity}
\end{table}

The pass@16 evidence of~\S\ref{sec:experiments:main} is an indirect measure of diversity, in that it counts whether \emph{some} of the $16$ rollouts are correct. To probe diversity directly at the strategy level, we run an independent LLM-as-judge evaluation on the $60$ problems of AIME 2024 and AIME 2025, generate $8$ rollouts per problem per method, and ask GPT-4o-mini to score the diversity of the resulting strategy mix on the $1$-to-$5$ Likert scale of~\citet{flowrl2025}; the prompt is reproduced verbatim in App.~\ref{app:exp:diversity}. Table~\ref{tab:diversity} reports the resulting per-method scores on Qwen2.5-7B.

\noindent Table~\ref{tab:diversity} admits two complementary readings. First, the untrained backbone scores highest in absolute terms, which is consistent with the fact that pretraining without RL has not yet concentrated probability mass on a small number of high-reward modes; this is the regime that all four RL methods are pushing the policy out of, and a uniform increase in average reward will, in general, decrease this strategy-level diversity score relative to the backbone. Second, conditional on having entered an RL regime, \method retains the largest fraction of the backbone's diversity with the smallest drop from the backbone score, followed by FlowRL, with the two reward-maximization methods GRPO and GSPO substantially below. The ordering matches the design intent: distribution matching is closer to the backbone than reward maximization, and the offline-decoupled variant \method is closer than the online-coupled variant FlowRL. The diversity gap to GRPO/GSPO is on the order of $0.7$ points on the $1$-to-$5$ scale, or roughly $70\%$ of one Likert level, and is large in the framework that the judge prompt itself sets up, where each Likert level corresponds to a discrete change in the number of distinct strategies present among the rollouts of a problem.

\subsection{Distillation-then-SFT vs.\ Distribution-Matching RL}
\label{sec:discussion:sft}

\begin{table}[tb!]
\centering
\small
\vspace{-10pt}
\caption{\textbf{SFT vs.\ \method on the same Stage~1 teacher trajectories.} Mean@8 across six math benchmarks for Qwen2.5-7B and Qwen3-4B-Base. Superscripts in green/red give the absolute delta vs.\ the Vanilla baseline within each backbone block; \textbf{bold} marks the best within-block value.}
\vspace{-10pt}
\setlength{\tabcolsep}{4pt}
\providecommand{\dup}[1]{\textsuperscript{\textcolor{green!55!black}{$+$#1}}}
\providecommand{\ddown}[1]{\textsuperscript{\textcolor{red!75!black}{$-$#1}}}
\begin{tabular}{l|l|llllll}
\toprule
Method & Avg. & AIME24 & AIME25 & AMC23 & MATH500 & Minerva & Olympiad\\
\midrule
\rowcolor{black!8}
\multicolumn{8}{l}{\textit{Backbone: Qwen2.5-7B}}\\
\;\; Vanilla
    & 23.4 & 6.25 & 2.92 & 30.0 & 54.7 & 23.2 & 23.2\\
\;\; \,+\,SFT
    & 27.8\dup{4.4}
    & 6.33\dup{0.1}
    & 10.0\dup{7.1}
    & 42.5\dup{12.5}
    & 60.1\dup{5.4}
    & 21.0\ddown{2.2}
    & 26.9\dup{3.7}\\
\;\; \,+\,\method, ours
    & \textbf{33.5}\dup{10.1}
    & \textbf{10.0}\dup{3.8}
    & \textbf{13.3}\dup{10.4}
    & \textbf{52.5}\dup{22.5}
    & \textbf{65.4}\dup{10.7}
    & \textbf{27.9}\dup{4.7}
    & \textbf{31.9}\dup{8.7}\\
\midrule
\rowcolor{black!8}
\multicolumn{8}{l}{\textit{Backbone: Qwen3-4B-Base}}\\
\;\; Vanilla
    & 29.7 & 7.92 & 6.25 & 42.6 & 64.4 & 26.7 & 30.1\\
\;\; \,+\,SFT
    & 47.5\dup{17.8}
    & 42.0\dup{34.1}
    & 32.7\dup{26.5}
    & 60.0\dup{17.4}
    & \textbf{74.9}\dup{10.5}
    & 30.6\dup{3.9}
    & 45.0\dup{14.9}\\
\;\; \,+\,\method, ours
    & \textbf{61.3}\dup{31.6}
    & \textbf{60.8}\dup{52.9}
    & \textbf{50.8}\dup{44.6}
    & \textbf{94.4}\dup{51.8}
    & 74.6\dup{10.2}
    & \textbf{32.9}\dup{6.2}
    & \textbf{54.0}\dup{23.9}\\
\bottomrule
\end{tabular}
\label{tab:sft}
\end{table}

\vspace{-2pt}

A natural concern about any pipeline that draws trajectories from a strong proposal is whether the RL stage adds value on top, or whether one could simply fine-tune $\pref$ on the high-reward subset of those trajectories and skip RL. We ablate by SFT-LoRA on the reward-$1$ subset of the same Stage~1 trajectories, with full hyperparameters in App.~\ref{app:exp:sft}; the two pipelines thus consume identical offline budgets and differ only in whether the trajectories are used as a supervision target or as an importance-sampling vehicle for the offline partition function of \method. Table~\ref{tab:sft} shows that SFT does improve over the backbone, but \method delivers a further $5.7$ Mean@8 points on Qwen2.5-7B and $13.8$ on Qwen3-4B-Base---i.e., on the same teacher trajectories, fitting the soft reward-tilted distribution recovers landscape information that the hard $\mathbf{1}\{r=1\}$ SFT target discards. This gap is particularly relevant to the diversity question of~\S\ref{sec:discussion:diversity}, since SFT on a single high-reward trajectory per prompt is a standard route to a peaked policy.

\section{Conclusion}
\label{sec:conclusion}

We introduced \method, a distribution-matching RL method that estimates the prompt-conditional partition function offline by importance sampling, amortizes the estimates into a frozen regressor, and substitutes this regressor for the online co-trained partition module at the policy-learning stage. The construction preserves the reward-tilted fixed point while strictly separating partition-function estimation from policy learning in data, gradients, loss, and diagnostics. Residual offline estimation error enters the policy objective only through a bounding envelope around the trajectory-balance loss. Empirically on math and code benchmarks, \method matches or exceeds the online-coupled baseline on average accuracy, is competitive with or ahead of reward-maximization baselines on the diversity-relevant multi-sample pass rate, and retains the largest fraction of the reference policy's strategy-level diversity among the methods compared. Furthermore, sensitivity studies on the proposal model and the inverse temperature track the bias and variance predictions of the analysis. For future work, we plan to extend \method to domains lacking a strong initial proposal by exploring synergistic approaches, such as combining iterative self-distillation methods with DISA to progressively bootstrap both the partition estimates and the policy exploration.

\bibliographystyle{plainnat}
\bibliography{bib/references}

\clearpage

\appendix
\section{Related Work}
\label{sec:related}

\paragraph{Reinforcement learning via reward maximization.} Proximal Policy Optimization (PPO)~\citep{ppo2017} and Group Relative Policy Optimization (GRPO)~\citep{grpo2024} dominate practice; GRPO replaces PPO's value baseline with an in-batch group baseline and underlies recent reasoning systems~\citep{deepseekr12025,kimik25_2025,dapo2025,gspo2025,srpo2025,reinforceplusplus2025,raft2023,bond2024,drgrpo2025}. Diversity-aware variants~\citep{dragrpo2025,entropymechanism2025,eighty20rule2025,reasonexplore2025} reshape the advantage to retain spread, while preference-based methods Direct Preference Optimization (DPO)~\citep{dpo2023} and Identity Preference Optimization (IPO)~\citep{ipo2024} cancel the per-prompt partition function via pairwise log-ratios over a KL-regularized reward-maximization objective. The shared property of this family is that the policy follows the gradient of the underlying Boltzmann optimum without ever instantiating it as an explicit fitting target. The reported consequence is concentration of probability mass on a single mode~\citep{prolongedrl2025,flowrl2025,dragrpo2025,entropymechanism2025}, which advantage-reshaping variants reduce but do not by construction remove; pairwise log-ratio methods optimize an ordinal preference signal rather than the reward-tilted distribution itself.

\paragraph{Reinforcement learning via distribution matching.} Distribution-matching methods make the reward-shaped target distribution explicit and train the policy to fit it, rather than only following a scalar reward gradient. The idea originates in GFlowNets~\citep{bengio2021flow} and trajectory balance~\citep{malkin2022trajectory}, where a flow network is trained so that terminal probability is proportional to reward. \citet{gflownetllm2024} apply GFlowNets to LLM fine-tuning, Flow of Reasoning~\citep{yu2024flow} extends this to multi-step reasoning, asynchronous trajectory balance~\citep{tba2025} decouples data collection from training, hybrid-flow systems~\citep{hybridflow2024} integrate flow objectives with policy gradients, and flow-matching policy gradients~\citep{flowmatchingpg2025} cast policy improvement as continuous flow matching. FlowRL~\citep{flowrl2025} is a representative LLM-scale instantiation, combining trajectory balance with length normalization, group-normalized rewards, and a PPO-style importance-sampling correction. A separate branch, LAD~\citep{lad2026}, sidesteps the partition function by minimizing an $f$-divergence against an advantage-induced distribution whose probability is proportional to within-group advantage rather than to the exponential of reward. Across the reward-tilted variants the partition function is co-trained with the policy on shared online rollouts, which entangles partition-function error in the policy gradient, leaves no separable quality signal, and forces re-fitting per run; FlowRL additionally reports sensitivity to off-policy drift and to large inverse-temperature~\citep{flowrl2025}.

\section{Algorithm}
\label{app:method:algorithm}

\begin{algorithm}[htbp]
\caption{\method: offline IS estimation, amortized denoising, and frozen-anchored RL.}
\label{alg:decrl}
\small
\begin{algorithmic}[1]
\State \textbf{Inputs:} reference $\pref$, proposal $\pteach$, reward transform $\tilde r$, inverse temperature $\beta$, offline budget $\Dz$, training set $\Dtr$, samples-per-prompt $N$.
\Statex \textbf{// Stage 1 (offline): IS estimation of $\log Z(q)$.}
\For{each $q\in\Dz$}
  \State sample $\{o_i\}_{i=1}^{N}\sim\pteach(\cdot\mid q)$; score rewards and record $\log\pref(o_i\!\mid\! q)$, $\log\pteach(o_i\!\mid\! q)$.
  \State compute transformed rewards $\tilde r(q,o_i)$; in implementation, use the finite-batch normalized plug-in counterpart.
  \State $\log\Zh_{\mathrm{IS}}(q)\gets\Lis_i\!\big(\log\pref-\log\pteach+\beta \tilde r(q,o_i)\big)-\log N$.
\EndFor
\Statex \textbf{// Stage 2 (offline): amortized denoising.}
\State fit $g_{\psi^\star}\gets\arg\min_\psi\,\E_{q\in\Dz}[(g_\psi(q)-\log\Zh_{\mathrm{IS}}(q))^2]$; freeze $\psi^\star$.
\Statex \textbf{// Stage 3 (online): RL with frozen anchoring.}
\State initialize $\theta\gets$ ref-policy parameters.
\For{each RL step}
  \State sample $q\in\Dtr$; rollout $\{o_i\}\sim\pold$; compute token-level PPO/GRPO ratios, transformed rewards, and any rollout weights $\omega_i$ used in the residual.
  \State query $g_{\psi^\star}(q)$ with stop-gradient on $\psi$; take a gradient step on $\mathcal{L}_{\method}(\theta)$ from \eqref{eq:m:decrl}.
\EndFor
\State \Return $\theta^\star$.
\end{algorithmic}
\end{algorithm}

This appendix also proves the formal statements used in Sec.~\ref{sec:method}. The proofs are stated for the trajectory-level TB residual in~\eqref{eq:bg:flowrl}--\eqref{eq:m:decrl} and for a fixed prompt-conditional affine reward transform $\tilde r(q,o)=a_q r(q,o)+b_q$, $a_q>0$. Finite-batch group normalization in implementation is a plug-in approximation to this fixed transform and contributes to the partition-function error term analyzed in App.~\ref{app:proof:reg-bound}.

\section{Proofs}
\label{app:proofs}

\subsection{Linear-scale unbiasedness of \texorpdfstring{$\Zh_{\mathrm{IS}}$}{Z\_IS}}
\label{app:proof:unbiased}

\begin{proposition}[Linear-scale unbiasedness]\label{prop:unbiased}
Fix a prompt $q$ and a fixed transformed reward $\tilde r(q,\cdot)$. Suppose $\pteach(o\mid q)>0$ whenever $\pref(o\mid q)\exp(\beta\tilde r(q,o))>0$. Then the estimator in~\eqref{eq:m:is} satisfies $\E[\Zh_{\mathrm{IS}}(q)]=Z(q)$ for every sample size $N\geq 1$. If $\mathrm{Var}_{\pteach}(w(q,\cdot))<\infty$, the log-space estimator in~\eqref{eq:m:logsumexp} is Jensen-biased downward with second-order asymptotic bias $\mathrm{CV}^2(w)/(2N)$.
\end{proposition}

\begin{proof}
Let
\begin{equation*}
w(q,o) \;=\; \frac{\pref(o\mid q)}{\pteach(o\mid q)}\,\exp(\beta\,\tilde r(q,o)),
\qquad
\Zh_{\mathrm{IS}}(q) \;=\; \frac1N\sum_{i=1}^{N}w(q,o_i),
\quad o_i\stackrel{\mathrm{i.i.d.}}{\sim}\pteach(\cdot\mid q).
\end{equation*}
By change of measure,
\begin{align*}
\E_{o\sim\pteach}[w(q,o)]
&= \sum_{o\in\mathcal{O}} \frac{\pref(o\mid q)}{\pteach(o\mid q)}\exp(\beta\tilde r(q,o))\pteach(o\mid q)\\
&= \sum_{o\in\mathcal{O}}\pref(o\mid q)\exp(\beta\tilde r(q,o))
= Z(q),
\end{align*}
where the ratio is well-defined on the support that matters by the absolute-continuity assumption. Linearity gives $\E[\Zh_{\mathrm{IS}}(q)]=Z(q)$.

For the log-space estimator, Jensen's inequality gives $\E[\log\Zh_{\mathrm{IS}}(q)]\leq \log Z(q)$. A second-order delta-method expansion of $\log x$ around $x=Z(q)$ yields
\begin{equation*}
\log Z(q)-\E[\log\Zh_{\mathrm{IS}}(q)]
\;\approx\; \frac{1}{2}\frac{\mathrm{Var}(\Zh_{\mathrm{IS}}(q))}{Z(q)^2}
\;=\; \frac{1}{2N}\frac{\mathrm{Var}_{\pteach}(w(q,\cdot))}{Z(q)^2}
\;=\; \frac{\mathrm{CV}^2(w)}{2N}.
\end{equation*}
\end{proof}

\paragraph{Remarks.}
Unbiasedness is a statement about the linear-scale estimator with fixed $\tilde r$. When $\tilde r$ is replaced by finite-batch normalization statistics computed from the same sampled group, the weights are no longer i.i.d. single-sample functions; the resulting plug-in error is part of the finite-sample partition-function error rather than the exact unbiasedness claim. The variance of $\Zh_{\mathrm{IS}}$ is $\mathrm{Var}_{\pteach}(w)/N$, so proposal quality controls the noise of the offline labels.

\subsection{Bias of the geometric-mean estimator}
\label{app:proof:gm}

\begin{proposition}[Geometric-mean bias]\label{prop:gm-bias}
For any non-degenerate $w(q,o)$ with $\E_{o\sim\pteach}[w]=Z(q)<\infty$ and $\E_{o\sim\pteach}[\log w]>-\infty$, define $\log\Zh_{\mathrm{GM}}(q)=\tfrac1N\sum_{i=1}^N\log w(q,o_i)$. Then
\begin{equation*}
\E_{o\sim\pteach}[\log w(q,o)] \;\leq\; \log\E_{o\sim\pteach}[w(q,o)] \;=\; \log Z(q),
\end{equation*}
with equality iff $w(q,o)$ is constant $\pteach$-a.s. Consequently $\E[\log\Zh_{\mathrm{GM}}(q)]\leq\log Z(q)$, with a second-order bias of order $\mathrm{CV}^2(w)/2$ that does not vanish with $N$.
\end{proposition}

\begin{proof}
The first inequality is Jensen's inequality for the concave function $\log$. Taking expectation in $\log\Zh_{\mathrm{GM}}(q)=\tfrac1N\sum_i\log w(q,o_i)$ gives
\begin{equation*}
\E[\log\Zh_{\mathrm{GM}}(q)] = \E_{o\sim\pteach}[\log w(q,o)] \leq \log Z(q).
\end{equation*}
Strictness follows from strict concavity unless $w$ is constant almost surely. A second-order expansion of $\log w$ around its mean gives
\begin{equation*}
\log Z(q)-\E[\log\Zh_{\mathrm{GM}}(q)]
\;\approx\; \frac{1}{2}\frac{\mathrm{Var}_{\pteach}(w(q,\cdot))}{Z(q)^2}
\;=\;\frac12\mathrm{CV}^2(w),
\end{equation*}
which is independent of $N$. Comparing with Prop.~\ref{prop:unbiased}, the logsumexp estimator reduces this second-order log bias by a factor of $N$.
\end{proof}

\subsection{Fixed point with a frozen partition function}
\label{app:proof:fixedpoint}

\begin{proposition}[Fixed-point family of \method]\label{prop:fixedpoint}
Fix a prompt $q$ and a fixed transformed reward $\tilde r(q,o)=a_qr(q,o)+b_q$ with $a_q>0$. Assume the policy class contains all full-support distributions over $\mathcal{O}$. Let
\begin{equation*}
R(o;\theta)=g_{\psi^\star}(q)+\log\ppol(o\mid q)-\log\pref(o\mid q)-\beta\tilde r(q,o),
\end{equation*}
and let $L(\theta\mid q;\mu)=\E_{o\sim\mu(\cdot\mid q)}[R(o;\theta)^2]$ for a full-support behavior distribution $\mu$.
\textit{(i) Exact partition function:} If $g_{\psi^\star}(q)=\log Z(q)$, then $\ppol(o\mid q)=\pref(o\mid q)\exp(\beta\tilde r(q,o))/Z(q)$ is the unique global minimizer of $L(\theta\mid q;\mu)$ for any $\theta$-independent full-support $\mu$.
\textit{(ii) Prompt-only partition-function bias:} If $g_{\psi^\star}(q)=\log Z(q)+\eta(q)$ with $\eta(q)$ independent of $o$, then the same target is a stationary point of the on-policy gradient flow for $\E_{o\sim\ppol(\cdot\mid q)}[R(o;\theta)^2]$.
\end{proposition}

\begin{proof}
Define $\tilde\pi_{\tilde r}(o\mid q)=\pref(o\mid q)\exp(\beta\tilde r(q,o))/Z(q)$.

For part~(i), substituting $g_{\psi^\star}=\log Z(q)$ gives
\begin{equation*}
R(o;\theta)=\log\ppol(o\mid q)-\log\tilde\pi_{\tilde r}(o\mid q).
\end{equation*}
Thus $L(\theta\mid q;\mu)\geq 0$, with equality iff $\log\ppol(o\mid q)=\log\tilde\pi_{\tilde r}(o\mid q)$ for $\mu$-almost every $o$. Full support and normalization imply $\ppol=\tilde\pi_{\tilde r}$. Since $\tilde r=a_qr+b_q$, the target is proportional to $\pref e^{\beta a_q r}$; the prompt-only factor $e^{\beta b_q}$ is absorbed by $Z(q)$.

For part~(ii), at $\ppol=\tilde\pi_{\tilde r}$ the residual is the constant $R(o;\theta)=\eta(q)$. The gradient of the on-policy objective is
\begin{equation*}
\nabla_\theta \E_{o\sim\ppol}[R(o;\theta)^2]
=\E_{o\sim\ppol}[R(o;\theta)^2\nabla_\theta\log\ppol(o\mid q)]
+\E_{o\sim\ppol}[2R(o;\theta)\nabla_\theta R(o;\theta)].
\end{equation*}
Because $\nabla_\theta R=\nabla_\theta\log\ppol$, evaluating at $\ppol=\tilde\pi_{\tilde r}$ gives
\begin{equation*}
(\eta(q)^2+2\eta(q))\,\E_{o\sim\tilde\pi_{\tilde r}}[\nabla_\theta\log\ppol(o\mid q)] = 0
\end{equation*}
by the score-function identity. Hence the target is stationary under on-policy sampling.
\end{proof}

\subsection{Regression-error bound}
\label{app:proof:reg-bound}

\begin{proposition}[Cauchy--Schwarz envelope on partition-function-error penalty]\label{prop:reg-bound}
Let $\eta(q)=g_{\psi^\star}(q)-\log Z(q)$ and $\sigma_g^2=\E_{q\sim\Dtr}[\eta(q)^2]$. Consider the formal trajectory-level residual $R(q,o;\theta)=\log Z(q)+\log\ppol(o\mid q)-\log\pref(o\mid q)-\beta\tilde r(q,o)$ with nonnegative rollout weights whose conditional expectation sums to one. Then, for every $\theta$,
\begin{equation*}
\big|\mathcal{L}_{\method}(\theta)-\mathcal{L}_{\mathrm{TB}}(\theta;Z)\big|
\;\leq\; \sigma_g^2 + 2\sigma_g\sqrt{\mathcal{L}_{\mathrm{TB}}(\theta;Z)}.
\end{equation*}
If an implementation uses additional bounded rollout weights, the same argument gives the same form up to constants determined by the weight bound.
\end{proposition}

\begin{proof}
Writing the rollout weights as $\alpha_i\geq0$ with $\E[\sum_i\alpha_i\mid q]=1$, the residual under the frozen partition function is $R_i+\eta(q)$. Therefore
\begin{equation*}
\mathcal{L}_{\method}-\mathcal{L}_{\mathrm{TB}}
=\E\!\Big[\sum_i\alpha_i\big(2R_i\eta+\eta^2\big)\Big].
\end{equation*}
The unit-mean condition gives
\begin{equation*}
\E\!\Big[\sum_i\alpha_i\eta(q)^2\Big]=\E_q[\eta(q)^2]=\sigma_g^2.
\end{equation*}
For the cross term, weighted Cauchy--Schwarz yields
\begin{equation*}
\Big|\E\!\Big[\sum_i\alpha_i R_i\eta\Big]\Big|
\leq
\sqrt{\E\!\Big[\sum_i\alpha_i R_i^2\Big]\,\E\!\Big[\sum_i\alpha_i\eta^2\Big]}
=\sigma_g\sqrt{\mathcal{L}_{\mathrm{TB}}(\theta;Z)}.
\end{equation*}
Combining the terms proves the bound. If the rollout weights are bounded but not exactly unit-mean, the same calculation multiplies the two terms by the corresponding weight-bound constants.
\end{proof}

\subsection{On-policy sampling and the bias counterexample}
\label{app:on-policy}
\label{app:counterexample}

\paragraph{Why on-policy sampling is essential for Proposition~\ref{prop:fixedpoint}(ii).} Part~(ii) relies on on-policy sampling: the score-function identity $\E_{\ppol}[\nabla\log\ppol]=0$ eliminates the prompt-only bias $\eta(q)$. Off-policy with a fixed $\mu\neq\ppol$, the analogous gradient at $\ppol=\tilde\pi$ contains $2\eta(q)\,\E_\mu[\nabla\log\ppol]$, which is generally nonzero. Thus a biased frozen partition function preserves stationarity only in the on-policy sense; exact partition-function recovery is required for the global-minimizer statement.

\paragraph{Why a zero-mean residual is insufficient.} A prompt-only bias does not preserve the global minimizer in general, even if such biases average to zero across prompts. The example below shows that the correct invariant for biased frozen partition functions is on-policy stationarity, not global optimality.

Fix one prompt $q$ with output space $\mathcal{O}=\{0,1\}$. Take $\pref(0\mid q)=\pref(1\mid q)=1/2$, reward $r(q,0)=\log 2$, reward $r(q,1)=0$, and $\beta=1$. Then $\tilde\pi(0\mid q)=2/3$, $\tilde\pi(1\mid q)=1/3$, and $\log Z(q)=\log(3/2)$. Suppose the regressor returns $g_{\psi^\star}(q)=\log Z(q)+\eta(q)$ with $\eta(q)=-2$.

The per-trajectory residual is
\begin{equation*}
R(o;\theta)=\eta(q)+\log\ppol(o\mid q)-\log\tilde\pi(o\mid q).
\end{equation*}
Parameterize $\ppol(0\mid q)=p$. Direct computation of the on-policy loss $L(p)=\E_{o\sim\ppol}[R(o;\theta)^2]$ gives
\begin{align*}
L(2/3) &= \eta^2 = 4,\\
L(0.9) &= 0.9\big(-2+\log\tfrac{0.9}{2/3}\big)^2 + 0.1\big(-2+\log\tfrac{0.1}{1/3}\big)^2 \approx 3.63 < 4.
\end{align*}
Thus a policy different from $\tilde\pi$ achieves lower on-policy squared residual than $\tilde\pi$ itself. The mechanism is the cross term
\begin{equation*}
L(p)=\E_{\ppol}\big[(\log\ppol-\log\tilde\pi)^2\big]+2\eta(q)\KL(\ppol\,\|\,\tilde\pi)+\eta(q)^2,
\end{equation*}
which can decrease the loss when $\eta(q)<0$. This does not contradict Prop.~\ref{prop:fixedpoint}: the gradient at $\ppol=\tilde\pi$ still vanishes under on-policy sampling, but the point need not be a global minimizer when the partition function has prompt-only bias.

\section{Experiment Details}
\label{app:exp}

\subsection{Training and evaluation hyperparameters}
\label{app:exp:hparams}

The four methods \method, FlowRL, GRPO, and GSPO share the configuration in Table~\ref{tab:app:hparams:rl}; differences across methods are confined to the loss function and to the auxiliary partition function of FlowRL/\method.

\begin{table}[h]
\centering
\caption{Shared RL training and evaluation configuration.}
\label{tab:app:hparams:rl}
\small
\begin{tabular}{ll}
\toprule
\textit{Hardware and trainer} & \\
GPUs & $16$ $\times$ A100 80GB, over two nodes\\
Trainer & veRL~\citep{hybridflow2024}, with dynamic batch sizing\\
\midrule
\textit{Optimization} & \\
Optimizer / learning rate & AdamW / $1\!\times\!10^{-6}$\\
Global batch size & $128$ prompts\\
Rollouts per prompt, $G$ & $8$\\
\verb|max_prompt_length| & $2048$\\
\verb|max_response_length| & $16384$\\
Inverse-temperature $\beta$ for \method and FlowRL & $15$\\
\midrule
\textit{Evaluation for math and code} & \\
Sampling temperature & $0.6$\\
Top-$p$ & $0.95$\\
Response length & $32{,}768$\\
Math metric & Mean@8 with math-verify on boxed final answer\\
Code metric & pass@1 and pass@16\\
\bottomrule
\end{tabular}
\end{table}

\paragraph{Checkpoint selection.} For each method--backbone pair we save checkpoints at regular intervals throughout training and report the checkpoint that is best on a held-out validation subset of the training corpus, evaluated under the same Mean@8 or pass@1 protocol as the test benchmarks. The same selection protocol is applied to every method; we do not transfer the selection across methods.

\subsection{\method-specific configuration for Stages~1 and~2}
\label{app:exp:disa}

The offline pipeline of \method introduces three additional design choices on top of the shared RL configuration above: the proposal $\pteach$, the offline budget $|\Dz|$ and $N$, and the regressor $g_\psi$. Table~\ref{tab:app:hparams:disa} summarizes the values used in the main results. The proposal is queried with the same generation hyperparameters as evaluation, namely temperature $0.6$ and top-$p$ $0.95$; offline log-probabilities of $\pref$ and $\pteach$ are recomputed with each model in inference mode. We do not prune prompts on any per-prompt importance-weight statistic in the main results, so $\Dz=\Dtr$.

\begin{table}[h]
\centering
\caption{\method offline-stage configuration.}
\label{tab:app:hparams:disa}
\small
\begin{tabular}{ll}
\toprule
Proposal $\pteach$, default & Qwen3-235B-A22B-Instruct-2507~\citep{qwen3_2025}\\
Proposal $\pteach$, weaker-teacher ablation & Qwen3-4B-Instruct-2507~\citep{qwen3_2025}\\
Offline prompt budget $|\Dz|$ & $17{,}000$ for math, DAPO-17k-processed\\
Samples per prompt $N$ & $8$\\
Aggregator for $\log\Zh_{\mathrm{IS}}$ & \texttt{logsumexp}, eq.~\eqref{eq:m:logsumexp}\\
\midrule
Regressor $g_\psi$ architecture & two-layer MLP, hidden width $64$, scalar output\\
Regressor input & frozen prompt embedding\\
Regressor optimizer & Adam, learning rate $10^{-3}$\\
Regressor batch size & $1024$\\
Regressor training epochs & $80$\\
Regressor objective & least squares on $\log\Zh_{\mathrm{IS}}$ using eq.~\eqref{eq:m:reg}\\
\bottomrule
\end{tabular}
\end{table}

\subsection{Offline rollout budget: selection of $N$}
\label{app:exp:nstudy}

The on-policy group size $G=8$ in Stage~3 fixes a natural anchor for the offline rollout count $N$, the number of teacher samples used per prompt to form $\log\Zh_{\mathrm{IS}}$, but it does not by itself decide whether $N=8$ is large enough to make the offline estimator stable, or so small that doubling it would materially improve the Stage~2 regression target. We address this with a separate variance--bias study of $\Zh_{\mathrm{IS}}$ on a representative subset of the math training prompts, run independently of any RL training.

\paragraph{Subset construction.} We sample a $\sim\!500$-prompt subset of the math training corpus by stratified $k$-means clustering on prompt embeddings. The number of clusters is selected by the elbow method on the SSE curve in Figure~\ref{fig:app:elbow}: the SSE drop transitions from steep to gradual at $k=8$, which we therefore use as the cluster count, and we stratify-sample within each cluster so that the subset stays distributionally representative while keeping teacher inference cost manageable.

\begin{figure}[h]
\centering
\includegraphics[width=0.7\textwidth]{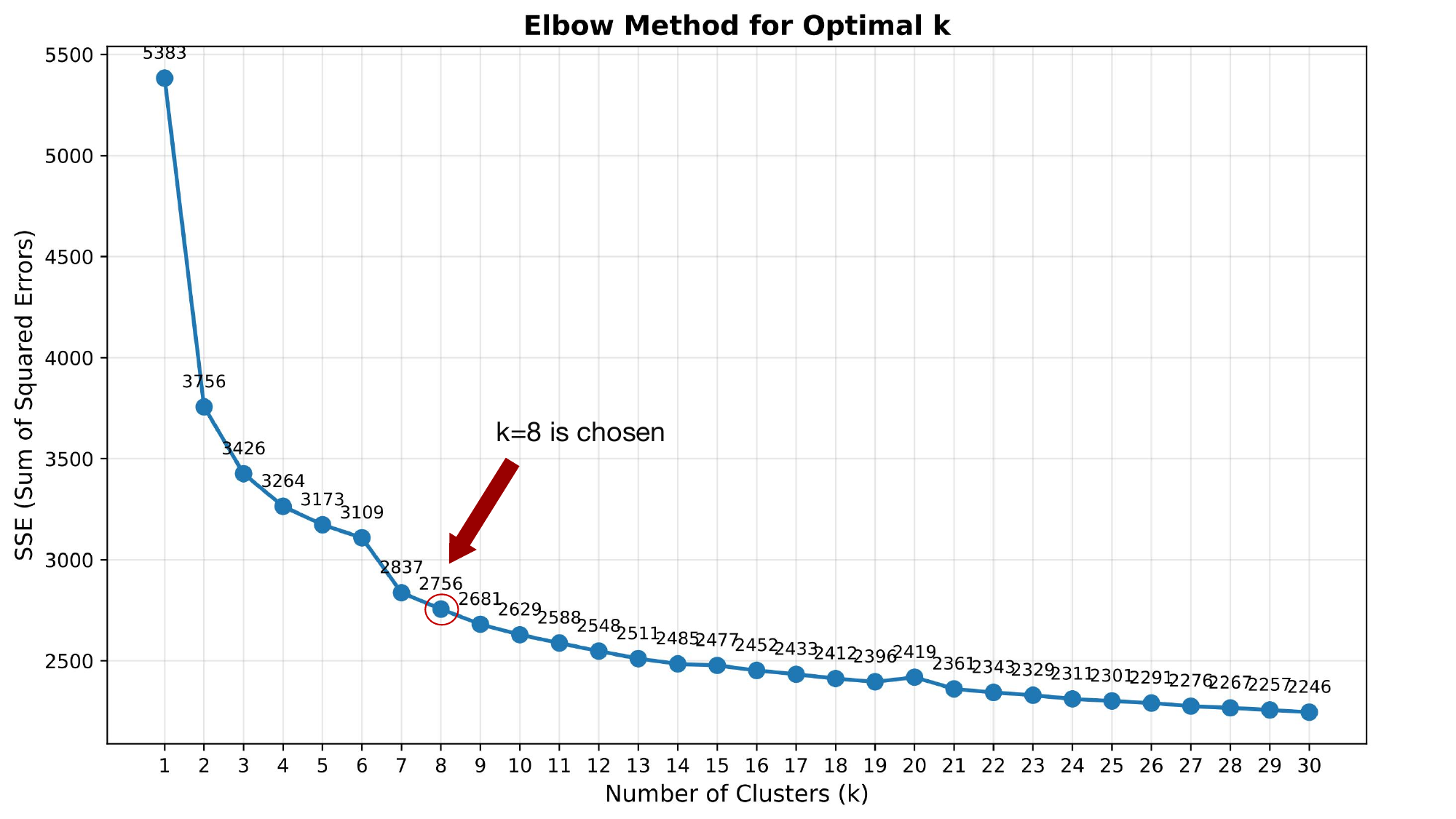}
\caption{\textbf{Cluster-count selection for the variance--bias subset.} SSE of $k$-means clustering on math-prompt embeddings versus the number of clusters $k$. The highlighted elbow at $k=8$ is the cluster count used for stratified sampling of the $\sim\!500$-prompt subset.}
\label{fig:app:elbow}
\end{figure}

\paragraph{Reference and subsampling protocol.} For each subset prompt $q$ we draw a high-precision pool of $32$ teacher rollouts and treat the pool-level estimate $\Zh_{32}(q)$ as a pseudo-ground-truth for $Z(q)$. We then subsample $M\in\{4,8,16\}$ rollouts without replacement and compute the corresponding offline estimate $\Zh_M(q)$, repeating the subsampling many times per prompt. We track two quantities as $M$ varies: the empirical variance of $\log\Zh_M(q)$ across subsamples, and the per-prompt relative bias $|\Zh_M(q)-\Zh_{32}(q)|/\Zh_{32}(q)$.

\paragraph{Results.} Figure~\ref{fig:app:nstudy:trend} reports the means $\pm$ standard deviations and the smoothed trend curves of both metrics. Both curves are monotone-decreasing in $M$ but flatten visibly around $M=8$: relative to $M=2$, the mean variance at $M=8$ has dropped by $\sim\!75\%$ and the median relative bias by $\sim\!50\%$, whereas the additional gain from $M=8$ to $M=16$ is only $\sim\!3$ percentage points of bias for double the teacher inference cost. The boxplots in Figure~\ref{fig:app:nstudy:box} confirm the same pattern at the distributional level: at $M\le 4$ the boxes are wide and dense outliers extend the relative bias up to $0.8+$, indicating highly unstable estimates on a non-trivial fraction of prompts; by $M=8$ the boxes have narrowed substantially and the extreme outliers have largely disappeared, with further compression from $M=8$ to $M=16$ being marginal.

\begin{figure}[h]
\centering
\includegraphics[width=0.7\textwidth]{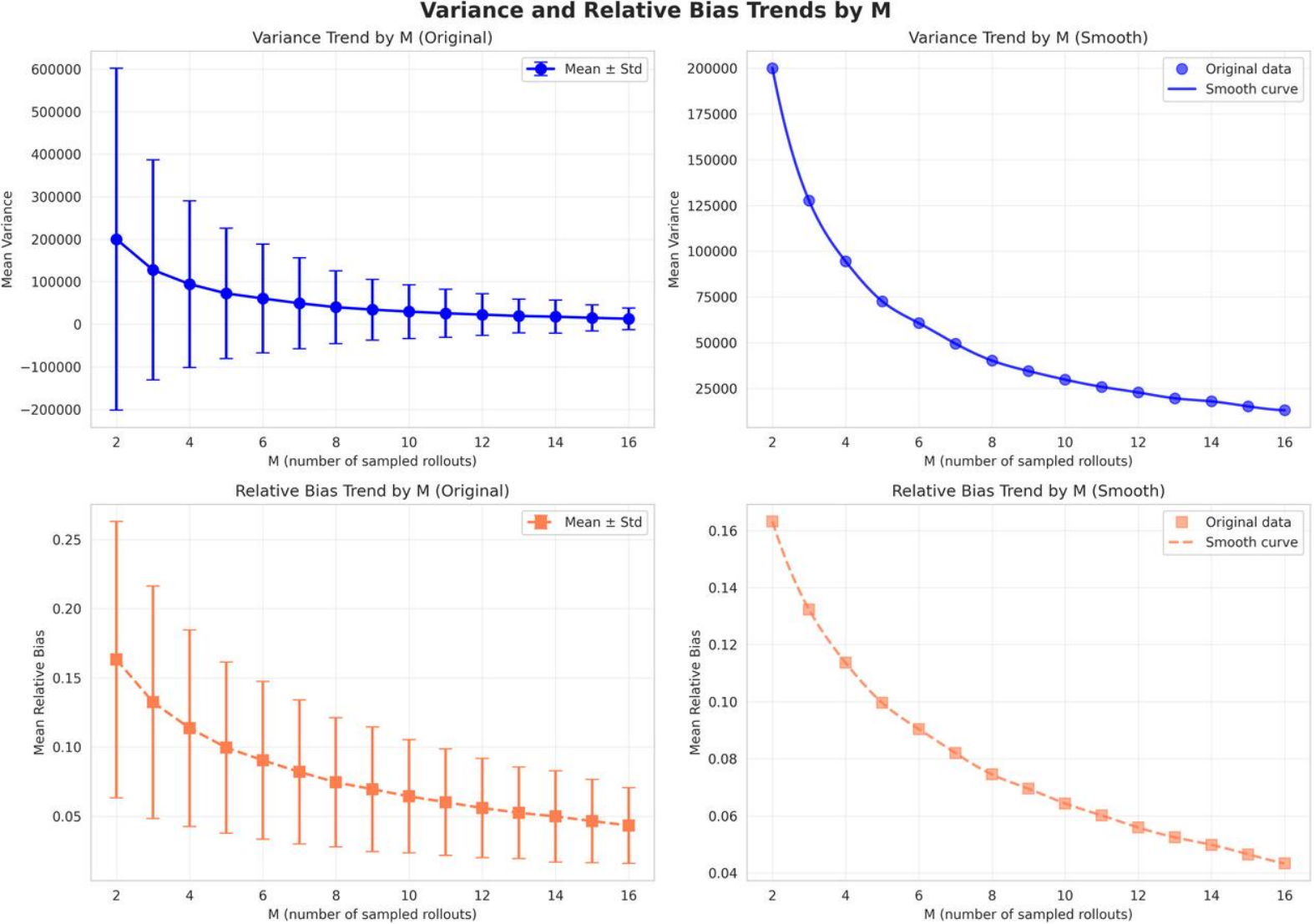}
\caption{\textbf{Variance and relative bias of the offline $\log Z$ estimator versus $M$.} Top panels: mean variance of $\log\Zh_M(q)$ across repeated subsamples, with $\pm$ standard deviation on the left and the smoothed trend on the right. Bottom panels: mean per-prompt relative bias $|\Zh_M-\Zh_{32}|/\Zh_{32}$, with $\pm$ standard deviation on the left and the smoothed trend on the right. Both curves exhibit a clear elbow around $M=8$, with diminishing returns thereafter.}
\label{fig:app:nstudy:trend}
\end{figure}

\begin{figure}[h]
\centering
\includegraphics[width=0.7\textwidth]{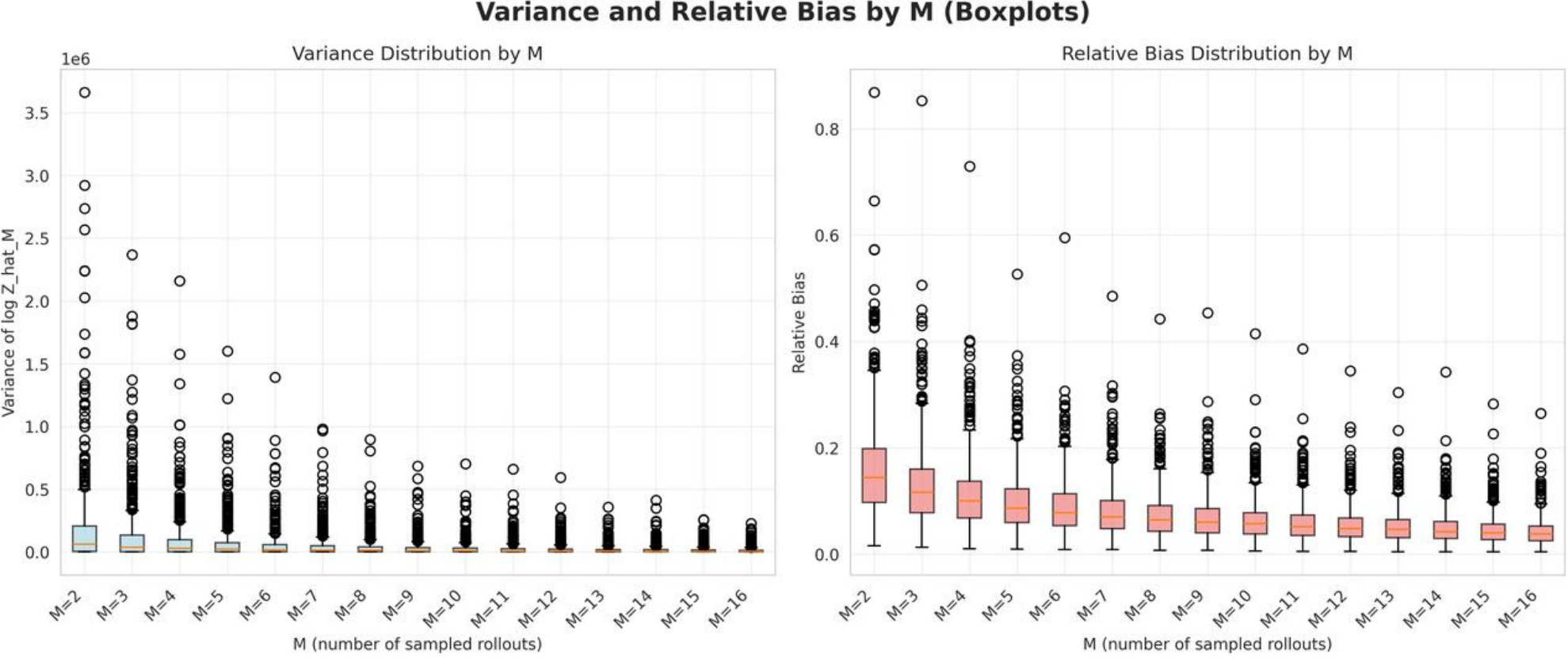}
\caption{\textbf{Distributions of variance and relative bias by $M$.} Boxplots of $\mathrm{Var}(\log\Zh_M)$ in the left panel and the relative bias $|\Zh_M-\Zh_{32}|/\Zh_{32}$ in the right panel across the $\sim\!500$ subset prompts. Dispersion narrows sharply between $M=2$ and $M=8$ and only marginally between $M=8$ and $M=16$; the extreme outliers visible at $M\le 4$ have largely disappeared by $M=8$.}
\label{fig:app:nstudy:box}
\end{figure}

\paragraph{Choice of $N$.} The two metrics together identify $M=8$ as an efficiency--accuracy sweet spot: variance and relative bias have entered their flat regimes while teacher inference cost is still half of $M=16$. We therefore set the offline rollout count to $N=8$ in the main pipeline, which has the additional convenience of matching the on-policy group size $G=8$ in Stage~3.

\subsection{SFT baseline details}
\label{app:exp:sft}

The SFT baseline of~\S\ref{sec:discussion:sft} reuses the Stage~1 trajectories of the math pipeline and is intended as a like-for-like comparison against an alternative use of the same offline asset. We filter the trajectories produced by $\pteach$ to those with verifier reward $r=1$, yielding approximately $120$k complete chain-of-thought solutions; each kept trajectory is paired with its prompt to form a question-answer corpus. We then run a LoRA fine-tune on $\pref$ with the configuration in Table~\ref{tab:app:sft}. Evaluation uses the same Mean@8 protocol as the main results, including the same temperature, top-$p$, and response length, so that the SFT and \method rows of Table~\ref{tab:sft} are directly comparable.

\begin{table}[h]
\centering
\caption{SFT baseline configuration.}
\label{tab:app:sft}
\small
\begin{tabular}{ll}
\toprule
Source data & Stage~1 trajectories with verifier $r=1$, $\approx 120$k CoT\\
Adapter & LoRA, rank $r=64$\\
Optimizer / learning rate & AdamW / $2\!\times\!10^{-5}$\\
Schedule / precision & cosine / bf16, DeepSpeed ZeRO-2\\
Global batch size & $128$\\
Epochs & $3$\\
\verb|max_seq_length| & $18{,}432$\\
Evaluation & Mean@8 with vLLM, temperature $0.6$, top-$p$ $0.95$\\
\bottomrule
\end{tabular}
\end{table}

\subsection{Diversity judge prompt}
\label{app:exp:diversity}

For the diversity evaluation in~\S\ref{sec:discussion:diversity} we use GPT-4o-mini as the judge model. For each problem and each method we generate $8$ rollouts and submit them to the judge with the prompt below, which is the diversity-judge prompt of~\citet{flowrl2025} reproduced verbatim. The judge returns a single integer in $\{1,2,3,4,5\}$; the resulting score is averaged over $60$ problems from AIME 2024 + AIME 2025 for each method, and the standard deviation in Table~\ref{tab:diversity} is the empirical standard deviation over the $60$ per-problem scores.

\begin{quote}\small
\textbf{System.} You are evaluating the DIVERSITY of solution approaches for a mathematics competition problem. Focus on detecting even SUBTLE differences in methodology that indicate different problem-solving strategies.

\textbf{User.} PROBLEM: \texttt{\{problem\}}.\\
$8$ SOLUTION ATTEMPTS: \texttt{\{formatted\_responses\}}.

\textbf{Evaluation criteria.} Rate diversity from $1$ to $5$:\\
\textit{Score 1 -- Minimal Diversity:} $14$+ responses use essentially identical approaches; same mathematical setup, same variable choices, same solution path; only trivial differences (arithmetic, notation, wording); indicates very low exploration/diversity in the generation process.\\
\textit{Score 2 -- Low Diversity:} $11$--$13$ responses use the same main approach; $1$--$2$ alternative approaches appear but are rare; minor variations within the dominant method (different substitutions, orderings); some exploration but heavily biased toward one strategy.\\
\textit{Score 3 -- Moderate Diversity:} $7$--$10$ responses use the most common approach; $2$--$3$ distinct alternative approaches present; noticeable variation in problem setup or mathematical techniques; balanced mix showing reasonable exploration.\\
\textit{Score 4 -- High Diversity:} $4$--$6$ responses use the most common approach; $3$--$4$ distinct solution strategies well-represented; multiple mathematical techniques and problem framings; strong evidence of diverse exploration strategies.\\
\textit{Score 5 -- Maximum Diversity:} no single approach dominates ($\leq 3$ responses use the same method); $4$+ distinctly different solution strategies; wide variety of mathematical techniques and creative approaches; excellent exploration and generation diversity.

IMPORTANT: focus on the DIVERSITY of the attempted approaches. Return ONLY a number from $1$ to $5$.
\end{quote}

The Likert thresholds in the prompt are written for $16$ rollouts as in the FlowRL setup; we adopt the same prompt verbatim while submitting $8$ rollouts per problem, so the score is calibrated against the FlowRL-style rubric. This choice keeps the diversity numbers in Table~\ref{tab:diversity} on a scale that is directly comparable to those reported by~\citet{flowrl2025} on the same benchmark family.

\subsection{Proposal-strength ablation: per-benchmark deltas and Val-MSE reading}
\label{app:exp:weak-teacher-details}

The proposal-strength ablation in~\S\ref{sec:experiments:robust} swaps $\pteach$ from Qwen3-235B-A22B-Instruct-2507 to the $\sim\!60\times$ smaller Qwen3-4B-Instruct-2507 and reruns Stages~1--3 on Qwen3-4B-Base, with every other hyperparameter held fixed. The Mean@8 averages reported in Figure~\ref{fig:weakteacher}(a) are: untrained backbone $29.7$, weaker proposal $37.1$, default proposal $65.5$. The weaker proposal therefore beats the untrained backbone by $7.4$ points but gives up most of the $35.8$-point gain delivered by the default proposal, and two of six per-benchmark deltas, on MATH-500 and Minerva, go negative.

Figure~\ref{fig:weakteacher}(b) reports validation MSE versus training epoch for the Stage~2 regressor under both proposals. Both teachers drive Val MSE to terminal values of order $5\!\times\!10^{-3}$ at the selected checkpoints, so $g_\psi$ fits its IS labels faithfully in either regime. We therefore read Val MSE as an auxiliary convergence diagnostic alongside Mean@8: a normalizer too noisy to support stable RL would have surfaced first as a non-converging or persistently large loss, which is not what we observe. The downstream accuracy gap consequently traces to the variance of the IS labels being fit, not to a failure of the regression to fit them.

\subsection{Inverse-temperature ablation: per-benchmark numbers}
\label{app:exp:beta-details}

The $\beta$ sweep in~\S\ref{sec:experiments:robust} runs \method on Qwen2.5-7B code at $\beta\in\{10,15,20\}$. Numerical details behind Figure~\ref{fig:beta}:

\begin{itemize}\itemsep2pt
\item \textbf{$\beta=10$ vs.\ default $\beta=15$.} Average pass@1 drops slightly from $31.80$ to $30.91$, with $\beta=10$ in fact slightly exceeding the default on LiveCodeBench and CodeForces; average pass@16 drops by $5.4$ points, from $45.3$ to $39.9$.
\item \textbf{$\beta=20$ vs.\ default $\beta=15$.} Average pass@1 collapses to $13.75$, an $\sim\!18$-point drop driven mainly by HumanEval+, where it falls from $71.2$ to $24.1$. Average pass@16 falls to $28.3$, only $+2.5$ over the Vanilla backbone, with LiveCodeBench and CodeForces individually falling \emph{below} the backbone's pass@16.
\end{itemize}

The asymmetry between the under-tilting and over-tilting sides matches the predicted failure modes of~\S\ref{sec:experiments:robust}: under-tilting yields a near-plateau in which mild signal loss in $\hat r_i$ is mostly absorbed by the regression, whereas over-tilting runs into a variance cliff once $\mathrm{CV}^2(w)$ becomes large enough that the LSE estimator is dominated by a few high-weight samples, with diversity-relevant pass@16 especially sensitive to the resulting mode collapse.


\end{document}